\pdfoutput=1
\documentclass[11pt,letterpaper]{article}

\usepackage[left=1in,right=1in,top=1in,bottom=0.9in]{geometry}

\usepackage{graphicx}
\usepackage{float}
\usepackage[table,xcdraw]{xcolor}
\usepackage{wrapfig,lipsum,booktabs}
\usepackage{setspace}
\usepackage{tikz}
\usepackage{booktabs}
\usepackage{tabularx}
\usepackage{array}
\usepackage{caption}
\usetikzlibrary{decorations.shapes,shapes.geometric,shadows,arrows,automata,positioning,calendar,mindmap,backgrounds,scopes,chains,er,3d,matrix,fit}

\usepackage{wrapfig}

\usepackage[utf8]{inputenc}
\usepackage{amsmath}
\usepackage{amsthm}
\usepackage{amsfonts,times}
\usepackage{cite}

\usepackage{tikz,lipsum,lmodern}
\usepackage[most]{tcolorbox}

\usepackage{amssymb}
\usepackage{url}
\usepackage{comment}

\usepackage[compact]{titlesec}
    \titlespacing{\section}{0pt}{3ex}{2ex}
    \titlespacing{\subsection}{0pt}{2.0ex}{1.0ex}
    \titlespacing{\subsubsection}{0pt}{1.0ex}{0.5ex}

\usepackage{array}
\newcolumntype{L}{>{\centering\arraybackslash}m{3cm}}

\newtheorem{theorem}{Theorem}
 \newtheorem{definition}[theorem]{Definition}

\newtheorem{problem}[theorem]{Problem}
\newtheorem{proposition}[theorem]{Proposition}

\title{\bfseries AI-Native Insurance for Agentic AI: Pricing, Underwriting, and End-to-End Automation}
\author{Quanyan Zhu\thanks{Quanyan Zhu is with the Department of Electrical and Computer Engineering and the NYU Center for Cybersecurity, NYU Tandon School of Engineering, New York University, Brooklyn, NY 11201, USA. Contact: \texttt{qz494@nyu.edu}.}}
\date{}

\begin{document}

\setstretch{1.67}

\pagenumbering{gobble}
\maketitle

\begin{abstract}
Agentic AI shifts insurance from covering passive digital assets to covering operational actors that reason, invoke tools, alter external state, and depend on shared model and infrastructure providers. These capabilities generate losses from cyber compromise, autonomous decision error, model drift, dependency outage, professional negligence, regulatory violation, and cyber-physical harm. This paper develops an AI-native framework for underwriting, pricing, and contract design for agentic-AI deployments. Each deployment is represented by a risk state \(s_i=(\alpha_i,\beta_i,\eta_i,g_i,v_i)\) capturing autonomy category, operational authority, external-state permissions, governance maturity, and dependency concentration. The framework maps this state into event probabilities, severities, governance costs, risk loadings, coverage incidence, indemnity allocation, and policy covenants, and then formulates a contract-design problem over premiums, deductibles, limits, aggregate exposure, allocation rules, and governance obligations subject to participation, insurer-profitability, and incentive-compatibility constraints. We show that insurability forms a region of the risk-state space, that fixed-terms feasibility deteriorates monotonically as exposure grows, and that a governance threshold certifies a deployment as insurable. The paper further reads insurance as an AI operating cost and regulatory-control instrument that internalizes risk, shapes adoption incentives, and supports mandatory financial-responsibility requirements for higher-risk deployments. A healthcare care-coordination case study solves a finite contract menu and shows how delegated authority, permission exposure, dependency concentration, and governance maturity drive feasibility and pricing, and an automated workflow illustrates how monitoring, trigger evaluation, claims validation, and human escalation operationalize the contract architecture.
\end{abstract}


\section{Introduction}
\label{sec:introduction}

Agentic artificial intelligence (AI) changes the unit of insurance analysis. Conventional cyber and technology policies are written around information systems, data assets, professional services, and security events. Agentic AI, by contrast, plans, invokes tools, coordinates multi-step workflows, communicates with external parties, modifies records, initiates transactions, and, in some deployments, acts on physical systems. The resulting exposure is no longer merely \emph{information risk}, in which a system produces an output that a human may accept or reject; it is \emph{action risk}, in which an AI-enabled workflow can directly create operational, legal, financial, or physical loss.

This distinction reshapes underwriting. The relevant questions are what the deployed agent is permitted to do, which systems it can access, whether execution requires human approval, whether actions are logged and reversible, which vendors or model providers the workflow depends on, and whether a loss would fall under cyber, technology errors and omissions (TechE\&O), professional liability, general liability, product liability, or an affirmative AI-specific endorsement. Market responses already point toward a layered architecture rather than a single monoline agentic-AI policy: cyber policies address security and privacy losses; TechE\&O and professional-liability forms address faulty services and negligent deployment; general-liability and product-liability forms address bodily injury and property damage; and emerging AI endorsements attempt to make silent AI exposure explicit \cite{bohme2010modeling,biener2015insurability,romanovsky2019content,axaxl2024genai,armilla2024affirmative,chaucer2025vanguard}.

The actuarial problem is hard because agentic-AI losses may occur without any traditional cyber compromise. A prompt-injection attack can redirect a legitimate workflow; a hallucinated clinical or financial recommendation can be converted into an external action; a model can drift after deployment; a cloud, model, or connector outage can affect many insureds at once; and a cyber-physical agent can cause bodily injury or property damage. Historical claims data for these mechanisms remain sparse, so pricing must initially rely on exposure assessment, scenario analysis, control effectiveness, dependency mapping, and accumulation-risk management---much as cyber insurance developed before broad actuarial credibility was available \cite{naic2025cyber,eiopa2024cyberstress,geneva2024accumulation}.

This paper develops a mathematical-programming framework for agentic-AI insurance contract design, building on recent work that treats agentic AI as an emerging insurable exposure class \cite{zhu2026insuranceagenticai}. The framework formalizes the practical underwriting variables---autonomy category, operational authority, external-state permissions, governance maturity tier, and dependency concentration---as a compact risk state. It maps event categories to coverage layers, defines a binary coverage-incidence matrix, separates coverage availability from payment allocation, and formulates an optimization problem that jointly determines pricing, deductibles, limits, aggregate exposure, allocation rules, and governance obligations. On this foundation we establish three structural results: insurability forms a region of the risk-state space, fixed-terms feasibility deteriorates monotonically as exposure grows, and a governance threshold certifies a deployment as insurable. The paper further reads insurance as an AI operating cost and regulatory-control instrument: bundled into deployment cost or mandated by law, coverage can internalize risk, shape adoption incentives, and impose financial responsibility on high-risk systems. A healthcare care-coordination case study translates operational facts into feasible insurance terms, and a workflow section operationalizes the design through continuous monitoring, trigger evaluation, claims validation, and human escalation for exceptional cases.

\section{Related Work}
\label{sec:related_work}

This paper draws on four strands of work. The first is AI safety and foundation-model risk. Early AI-safety research identifies accident mechanisms such as side effects, reward misspecification, distributional shift, unsafe exploration, and scalable-oversight failures \cite{amodei2016concrete}. Foundation-model research shows that broadly reusable models create cross-domain capabilities and cross-domain risks, including opacity, data-quality problems, bias, documentation gaps, and governance challenges \cite{bommasani2021foundation,bender2021stochastic}. Work on frontier AI risk further motivates auditability, rigorous evaluation, and adaptive governance as systems grow more capable and autonomous \cite{bengio2024extreme}. Together these studies justify treating governance and monitoring as insurance-relevant variables, not merely technical best practices.

The second strand is adversarial machine learning and prompt-based attacks. Adversarial-example research shows that machine-learning systems behave unpredictably under strategically chosen inputs \cite{goodfellow2015adversarial,carlini2017robustness}. In large-language-model applications, prompt injection and indirect prompt injection extend this concern to tool-using systems that retrieve external content, follow instructions, or execute actions \cite{perez2022ignore,greshake2023not,owasp2025llm}. The implication for insurance is that loss pathways include manipulated context, compromised tool outputs, and unsafe action execution---not only network intrusion or data breach.

The third strand is cyber-insurance economics and incentive design. Classic work on security investment treats insurance as both risk transfer and an incentive mechanism \cite{gordon2002economics,bolot2009cyber,bohme2010modeling}. Empirical studies of cyber-policy language and pricing show that coverage definitions, exclusions, limits, and claims categories decide whether insurance improves risk management or merely transfers poorly understood losses \cite{romanovsky2019content}. Contract theory under asymmetric information motivates the participation, incentive-compatibility, and moral-hazard constraints we adopt \cite{rothschild1976equilibrium,ehrlich1972market,raviv1979design}. Mechanism-design work by Liu and coauthors studies how incentives improve interdependent security investment and how aggregate cyber-risk dependence threatens the sustainability of insurers and reinsurers \cite{naghizadeh2015tale,pal2021aggregate}.

Several game-theoretic and contract-design papers by Zhu and coauthors are especially relevant. Attack-aware cyber-insurance models study interdependent networks in which users, attackers, and insurers interact strategically \cite{zhang2016attackaware}, and dynamic contract-design models capture self-protection, risk compensation, and evolving risk states \cite{zhang2021optimal}. Related work on systemic cyber-risk management and cyber resilience treats monitoring, residual risk, asymmetric information, and moral hazard as central design variables \cite{chen2021dynamic,liu2022moralhazard,liu2023cyberresilience}. Most directly, recent work frames underwriting, pricing, risk transfer, and claims design as core challenges for AI deployments, while work on the Internet of Agentic AI emphasizes communication, coordination, collective intelligence, and interdependence among agents at scale \cite{zhu2026insuranceagenticai,zhu2026internetagenticai}. This paper extends that perspective by linking autonomy, delegated authority, permissions, governance controls, and dependency concentration to insurance pricing, allocation, and claims execution.

\section{Agentic AI in Practice and Underwriting Observables}
\label{sec:agentic_ai_practice}

\begin{definition}[Agentic-AI deployment]
An agentic-AI deployment is an AI-enabled operational workflow in which a model is connected to tools, data sources, memory or task context, workflow logic, and execution interfaces, so that model-generated outputs may condition actions in an external environment. Such actions may include calling services, updating records, sending messages, initiating transactions, delegating subtasks, or interacting with physical systems.
\end{definition}

The insurance-relevant distinction is external-state change. An output-only system may produce misleading information, but an agentic deployment converts a model output into an operational step. Once output is wired to an action channel, errors, attacks, outages, and control failures become operational losses rather than merely informational defects \cite{bommasani2021foundation,nist2023airmf,owasp2025llm}.

In current practice, agentic AI is usually embedded inside business workflows rather than deployed as a stand-alone autonomous machine. Examples include customer-service and sales agents that retrieve account data and issue responses; software-engineering and security-operations agents that inspect code, open tickets, generate patches, or triage alerts; healthcare administrative agents that schedule visits, route patient messages, and summarize records; financial and compliance agents that review transactions or prepare filings; procurement and billing agents that initiate orders, refunds, or invoices; and cyber-physical agents that interact with devices, buildings, robots, vehicles, or industrial equipment. As these systems interconnect through shared communication, coordination, and service layers, deployment risk also depends on how each local agent participates in larger agentic-AI ecosystems \cite{zhu2026internetagenticai}. Such applications differ less by the label ``AI'' than by the operational boundary drawn around the deployed system: what it can access, what it may change, whether a human must approve execution, how actions are logged and reversed, and which model, cloud, data, or connector providers it depends on.

The same application category can therefore carry very different insurance exposures. A clinical assistant that drafts a message for clinician approval differs sharply from one that sends the message, schedules a follow-up visit, and updates the patient record without review. A customer-service agent that reads account data differs from one that issues refunds or changes billing records. A software agent that recommends code differs from one that merges code, rotates credentials, or modifies cloud infrastructure. These distinctions motivate the formal risk state of Section~\ref{sec:risk_coverage_framework}: \(\alpha_i\) captures capability, \(\beta_i\) operational authority, \(\eta_i\) permissions, \(g_i\) governance maturity, and \(v_i\) dependency concentration.

\begin{table}[ht]
\centering
\small
\caption{Agentic-AI Deployment Features and Insurance-Relevant Observables}
\label{tab:agentic_ai_observables}
\begin{tabularx}{\textwidth}{p{0.30\textwidth}p{0.36\textwidth}>{\raggedright\arraybackslash}X}
\hline
Underwriting question & Observable evidence & Risk-state component \\
\hline
What can the agent technically do? & Model and workflow description; task classes; tool manifests; orchestration design & Autonomy category \(\alpha_i\) \\
How much can execute without approval? & Approval policy; sampled action logs; autonomous-action frequency; escalation records & Operational authority \(\beta_i\) \\
Which systems can the agent touch? & API scopes; connectors; credentials; permission inventory; data and tool access lists & Permission vector \(\eta_i\) and permission exposure \(\psi_\eta(\eta_i)\) \\
How controlled and auditable is the deployment? & Audit logs; monitoring controls; testing cadence; incident response; rollback and change-management evidence & Governance tier \(g_i=h(\mathbf z_i)\) \\
How concentrated are model and service dependencies? & Model, cloud, API, data, and connector vendor inventory; traffic or workload shares & Dependency vector \(v_i\) and concentration \(R(v_i)\) \\
\hline
\end{tabularx}
\end{table}

Table~\ref{tab:agentic_ai_observables} also shows why a risk-state representation is useful for insurance. Every component is observable---from underwriting questionnaires, permission inventories, tool manifests, audit evidence, telemetry, vendor records, incident reports, and policy-covenant attestations. Once collected, the state \(s_i\) feeds the event-probability maps, severity maps, coverage incidence, governance covenants, pricing schedules, and claims logic. The formal model below is thus best read as a compact representation of how an agentic-AI deployment operates in practice.

\section{Risk-State and Coverage Framework for Agentic-AI Insurance}
\label{sec:risk_coverage_framework}

The preceding section described agentic AI as an operational system with action channels, permissions, approval rules, controls, and shared dependencies. The modeling thesis of this paper follows: the insured object is no longer an information system or software platform, but the joint exposure created by the AI system's autonomy, operational authority, external-state permissions, governance controls, and dependency ecosystem. Whereas traditional cyber insurance ties losses primarily to external attacks and security failures \cite{bohme2010modeling,biener2015insurability}, agentic-AI insurance must also cover losses generated by the AI system's own behavior---hallucinations, autonomous decision errors, prompt-injection attacks, model drift, dependency failures, and cyber-physical incidents \cite{amodei2016concrete,nist2023airmf,owasp2025llm}.

\subsection{From Underwriting Practice to Risk-State Variables}

Building on the observables in Table~\ref{tab:agentic_ai_observables}, the practical underwriting question is not whether an organization uses AI, but what the agent is authorized to do, which systems it can reach, what approvals precede execution, whether actions are logged and reversible, and how concentrated the deployment is across model, cloud, or connector providers. Agentic-AI underwriting should therefore begin with operational authority and control evidence rather than a generic AI-use questionnaire. The variables below are not abstract labels; each corresponds to information that an insurer can request in an application, verify through audit evidence, monitor through telemetry, and translate into pricing or coverage terms.

We use a fixed indexing convention throughout the model. The insured organization is indexed by \(i\in\mathcal I=\{1,\ldots,n\}\) and appears as a subscript. The loss event \(e\in\mathcal E\) and coverage layer \(r\in\mathcal R\) both appear as superscripts. Thus \(q_i^e\) denotes the probability of event \(e\) for insured \(i\), \(x_i^e\) denotes the corresponding gross loss, \(Y_i^{e,r}\) denotes the indemnity attributed to coverage layer \(r\), and \(Y_i^e=\sum_{r\in\mathcal R}Y_i^{e,r}\) denotes total indemnity for event \(e\) before any aggregate cap is applied.

The variable \(\alpha_i\) is an autonomy category, not a continuous maturity score. It records how far the insured system moves from information generation toward external action. The categories are ordinal: higher categories indicate greater loss-generating capability, but the gap between adjacent categories is not numerically equal. We write \(\mathcal A=\{\alpha^{(0)},\alpha^{(1)},\alpha^{(2)},\alpha^{(3)},\alpha^{(4)}\}\), ordered \(\alpha^{(0)}\prec\alpha^{(1)}\prec\alpha^{(2)}\prec\alpha^{(3)}\prec\alpha^{(4)}\); Table~\ref{tab:autonomy_categories} summarizes them.

\begin{table}[ht]
\centering
\caption{Autonomy Categories for Agentic-AI Underwriting}
\label{tab:autonomy_categories}
\begin{tabularx}{\textwidth}{c p{0.25\textwidth} X}
\hline
\(\alpha_i\) & Category & Practical underwriting meaning \\
\hline
\(\alpha^{(0)}\) & Assistive AI & Produces information, text, code, or recommendations; no independent external execution. \\
\(\alpha^{(1)}\) & Tool-enabled copilot & Retrieves information or prepares actions through tools, but material execution remains subject to human approval. \\
\(\alpha^{(2)}\) & Digital agent & Executes digital actions such as sending communications, updating records, browsing, scheduling, or initiating workflow changes. \\
\(\alpha^{(3)}\) & Multi-agent workflow & Coordinates multiple agents, delegates subtasks, and may create cascading digital actions across systems. \\
\(\alpha^{(4)}\) & Cyber-physical agent & Acts through robotics, industrial systems, medical devices, vehicles, buildings, or other physical infrastructure. \\
\hline
\end{tabularx}
\end{table}

The autonomy category \(\alpha_i\) describes capability---what the agent is technically able to do. The variable \(\beta_i\in[0,1]\) describes operational authority---how much of that capability may execute without human approval. Let \(\mathcal U_i\) denote the set of operational actions the deployed agent is authorized to initiate under its realized permission profile; \(\mathcal U_i(\eta_i)\) below makes this permission dependence explicit. For each action \(u\in\mathcal U_i\), define the human-approval indicator
\begin{equation}
h_i(u)
=
\begin{cases}
1, & \text{if execution of action \(u\) requires human approval},\\
0, & \text{if execution of action \(u\) may proceed autonomously}.
\end{cases}
\label{eq:human_approval_indicator}
\end{equation}
Let \(\omega_i(u)\) denote the relative frequency, criticality, or business importance of action \(u\), with \(\omega_i(u)\ge0\) and \(\sum_{u\in\mathcal U_i}\omega_i(u)=1\). The operational-authority variable is then
\begin{equation}
\beta_i
=
\sum_{u\in\mathcal U_i}
\omega_i(u)\bigl(1-h_i(u)\bigr),
\qquad
0\le \beta_i\le1.
\label{eq:operational_authority}
\end{equation}
Thus \(\beta_i=0\) means that every authorized action requires human approval, while \(\beta_i=1\) means that all authorized actions may execute autonomously. Intermediate values represent partial delegation.

Equation~\eqref{eq:operational_authority} admits a probabilistic reading. If an authorized action is sampled according to the weights \(\omega_i(\cdot)\) and \(H_i\) indicates that the sampled action requires approval, then \(\beta_i=\mathbb P(H_i=0)\) and \(1-\beta_i=\mathbb P(H_i=1)\). For telemetry-based underwriting over an observation period \([0,T]\), the same quantity is estimated from logs by \(\beta_i(T)=N_i^{\mathrm{auto}}(T)/N_i^{\mathrm{total}}(T)\), the ratio of actions executed without human intervention to all agent actions in the period. This makes \(\beta_i\) observable, auditable, and suitable for continuous underwriting.

\begin{table}[ht]
\centering
\caption{Autonomy Category Versus Operational Authority}
\label{tab:alpha_beta_examples}
\begin{tabularx}{\textwidth}{>{\raggedright\arraybackslash}p{0.36\textwidth}cc>{\raggedright\arraybackslash}X}
\hline
Deployment & \(\alpha_i\) & \(\beta_i\) & Interpretation \\
\hline
Clinical copilot drafting recommendations & \(\alpha^{(2)}\) & 0.05 & Digital capability with nearly all actions reviewed by a clinician \\
Appointment-scheduling agent & \(\alpha^{(2)}\) & 0.60 & Same capability class, but many scheduling actions execute autonomously \\
Autonomous customer-service workflow & \(\alpha^{(3)}\) & 0.85 & Multi-step workflow with limited human approval before execution \\
Industrial control agent & \(\alpha^{(4)}\) & 0.95 & Cyber-physical capability with very high autonomous execution authority \\
\hline
\end{tabularx}
\end{table}

The vector \(\eta_i\) is the agent's permission profile: it specifies which parts of the external environment the deployed agent is authorized to touch. Let \(\mathcal P=\{p_1,\ldots,p_m\}\) denote the universe of permission classes. The permission-allocation vector is \(\eta_i=(\eta_{i,1},\ldots,\eta_{i,m})\in\{0,1\}^m\), where
\begin{equation}
\eta_{i,j}
=
\begin{cases}
1, & \text{if the deployed agent is authorized to exercise permission \(p_j\)},\\
0, & \text{otherwise}.
\end{cases}
\label{eq:permission_indicator}
\end{equation}
Typical permission classes include external email, scheduling, customer-record modification, billing or payment initiation, procurement, electronic-health-record updates, physical-device control, and external API execution. If \(\mathcal P(u)\subseteq\mathcal P\) denotes the permissions required to execute action \(u\), then the permissions determine the admissible action set
\begin{equation}
\mathcal U_i(\eta_i)
=
\left\{
u:\mathcal P(u)\subseteq \{p_j\in\mathcal P:\eta_{i,j}=1\}
\right\}.
\label{eq:admissible_action_set}
\end{equation}
Thus \(\eta_i\) answers what the agent may do, while \(\beta_i\) answers how often authorized actions may proceed without human approval. For example, an agent may have email permission \(\eta_{i,\mathrm{email}}=1\) while every email still requires review, yielding \(\beta_i=0\) for an email-only action set; permission exists, but autonomous execution authority does not. If the same email actions execute automatically, then both permission and authority are present. If \(\eta_{i,\mathrm{email}}=0\), email actions are infeasible regardless of the value of \(\beta_i\).

The unweighted count \(\|\eta_i\|_1=\sum_{j=1}^m\eta_{i,j}\) reports how many permission classes are granted but treats them as equally risky. Because sending an email, modifying a payment record, and controlling a physical device carry very different loss implications, we use the weighted permission-exposure map \(\psi_\eta(\eta_i)=\sum_{j=1}^m\omega_j^\eta\eta_{i,j}\), where \(\omega_j^\eta>0\) is the risk weight of class \(p_j\) and the unweighted count is the special case \(\omega_j^\eta\equiv1\). Table~\ref{tab:permission_weights} gives representative classes and illustrative relative weights.

\begin{table}[ht]
\centering
\caption{Illustrative Permission Classes and Relative Risk Weights}
\label{tab:permission_weights}
\begin{tabularx}{\textwidth}{p{0.24\textwidth}p{0.16\textwidth}>{\raggedright\arraybackslash}X}
\hline
Permission class \(p_j\) & Weight \(\omega_j^\eta\) & Operational meaning \\
\hline
External email & 1 & Send external communications or notifications \\
Scheduling & 2 & Create, cancel, or modify appointments and workflow tasks \\
Record modification & 4 & Modify customer, patient, or operational records \\
Payments or billing & 8 & Initiate financial transactions, invoices, refunds, or billing changes \\
Physical-device control & 15 & Control equipment, robotics, medical devices, buildings, vehicles, or industrial systems \\
\hline
\end{tabularx}
\end{table}

For each insured organization \(i\), the underwriting state is
\begin{equation}
s_i=(\alpha_i,\beta_i,\eta_i,g_i,v_i)
\in
\mathcal S
:=
\mathcal A\times\mathcal B\times\mathcal P_m\times\mathcal G\times\mathcal V,
\label{eq:risk_state_space}
\end{equation}
where \(\mathcal B=[0,1]\), \(\mathcal P_m=\{0,1\}^m\), \(\mathcal G=\{g^{(1)},\ldots,g^{(M)}\}\), and \(\mathcal V=\Delta_K=\{v\in\mathbb R_+^K:\sum_{k=1}^Kv_k=1\}\). The governance tier set \(\mathcal G\) is finite and ordered as \(g^{(1)}\prec g^{(2)}\prec\cdots\prec g^{(M)}\). Here \(v_i\) records dependency shares across \(K\) model providers, cloud services, agent platforms, data sources, or connectors. A concentration statistic such as \(R(v_i)=\sum_{k=1}^Kv_{i,k}^2\) can be used when the insurer needs a scalar accumulation-risk measure.

The governance variable \(g_i\) is a maturity tier capturing the overall effectiveness of governance, monitoring, and operational safeguards. The order \(g^{(1)}\prec\cdots\prec g^{(M)}\) represents increasing assurance, but the tiers are ordinal and need not be equally spaced. To ground the tier in evidence, let \(\mathcal L_g=\{1,\ldots,L_g\}\) index governance-control dimensions. For each \(\ell\in\mathcal L_g\), let \(a_{i,\ell}^{\mathrm{ctrl}}\in\mathcal D_\ell\) be the audit evidence, telemetry record, questionnaire response, or technical assessment submitted by insured \(i\), and let a scoring rule \(\zeta_\ell:\mathcal D_\ell\to[0,1]\) return the normalized score \(z_{i,\ell}=\zeta_\ell(a_{i,\ell}^{\mathrm{ctrl}})\): zero for an absent or ineffective control, one for a fully implemented and evidenced control, and intermediate values for partial implementation or weak evidence. Writing \(\mathbf z_i=(z_{i,1},\ldots,z_{i,L_g})\in[0,1]^{L_g}\), the tier is assigned by the underwriting rule \(g_i=h(\mathbf z_i)\), \(h:[0,1]^{L_g}\to\mathcal G\). A higher tier means agent actions are more strongly constrained, observable, auditable, and recoverable. Representative dimensions include human approval gates, least-privilege access, credential isolation, execution logging, real-time monitoring, prompt-injection defenses, red-team testing, rollback capability, incident response, and vendor-risk management. The representation connects directly to practice: insurers can condition premium credits, deductibles, exclusions, or continued coverage on evidence supporting individual control components.

\begin{table}[ht]
\centering
\caption{Risk-State Variables and Domains}
\label{tab:risk_state_notation}
\begin{tabularx}{\textwidth}{p{0.12\textwidth}p{0.30\textwidth}>{\raggedright\arraybackslash}X}
\hline
Variable & Domain & Interpretation \\
\hline
\(\alpha_i\) & \(\mathcal A\) & Autonomy category of the AI deployment \\
\(\beta_i\) & \([0,1]\) & Operational authority measure: fraction of weighted authorized actions that may execute autonomously \\
\(\eta_i\) & \(\{0,1\}^m\) & Permission-allocation vector determining which external-state actions are feasible \\
\(g_i\) & \(\mathcal G=\{g^{(1)},\ldots,g^{(M)}\}\) & Governance maturity tier \\
\(v_i\) & \(\Delta_K\) & Dependency shares across \(K\) providers or services \\
\hline
\end{tabularx}
\end{table}

\subsection{Mapping Notation and Indexing Discipline}

The state vector \(s_i\) becomes useful for underwriting only after it is mapped into event probabilities, severity estimates, governance costs, and contract terms. To avoid ambiguity, we use uppercase symbols for mappings and lowercase indexed symbols for evaluated probability and severity quantities. For each event \(e\in\mathcal E\), let \(Q^e:\mathcal S\to[0,1]\) be the annual event-probability mapping and \(X^e:\mathcal S\to\mathbb R_+\) be the representative conditional gross-severity mapping. The realized inputs for insured \(i\) are \(q_i^e=Q^e(s_i)\) and \(x_i^e=X^e(s_i)\). Thus \(q_i^e\) and \(x_i^e\) are scalar event-level quantities for the fixed insured; they are not functions. If a full severity distribution is required, \(X^e(s_i)\) can be replaced by a conditional moment, quantile, or scenario statistic of a random severity \(X^e(s_i,\omega)\). In Table~\ref{tab:functional_mappings}, \(\mathcal C_i\) denotes the feasible contract set for insured \(i\), and \(g_i^\star\) denotes the governance tier required or selected in the issued contract.

The coverage-layer index \(r\) appears only when a quantity is layer specific. Event frequency and gross severity, \(q_i^e\) and \(x_i^e\), do not carry \(r\) because the event occurs before the claim is allocated to coverage layers. Deductibles \(D_i^r\), per-event limits \(L_i^r\), layer payments \(Y_i^{e,r}\), incidence indicators \(\gamma_i^{e,r}\), and allocation shares \(\lambda_i^{e,r}\) do carry \(r\) because they depend on the coverage layer. The total premium \(T_i\), aggregate AI limit \(A_i\), governance cost \(K_i(g_i)\), and total event indemnity \(Y_i^e=\sum_{r\in\mathcal R}Y_i^{e,r}\) do not carry \(r\) because they are contract-level, insured-level, or event-level quantities after aggregation.

\begin{table}[ht]
\centering
\caption{Functional Mappings and Realized Quantities}
\label{tab:functional_mappings}
\begin{tabularx}{\textwidth}{p{0.24\textwidth}p{0.32\textwidth}X}
\hline
Object & Mapping & Realized quantity for insured \(i\) \\
\hline
Event probability & \(Q^e:\mathcal S\to[0,1]\) & \(q_i^e=Q^e(s_i)\) \\
Gross severity & \(X^e:\mathcal S\to\mathbb R_+\) & \(x_i^e=X^e(s_i)\) \\
Permission exposure & \(\psi_\eta:\mathcal P_m\to\mathbb R_+\) & \(\psi_\eta(\eta_i)\) \\
Governance cost & \(K_i:\mathcal G\to\mathbb R_+\) & \(K_i(g_i)\) \\
Premium schedule & \(\tau_i:\mathcal G\to\mathbb R_+\) & \(T_i=\tau_i(g_i^\star)\) when used \\
Deductible schedule & \(d_i^r:\mathcal G\to\mathbb R_+\) & \(D_i^r=d_i^r(g_i^\star)\) when used \\
Risk loading & \(\varrho_i:\mathcal S\times\mathcal C_i\to\mathbb R_+\) & \(\varrho_i(s_i,C_i)\) \\
\hline
\end{tabularx}
\end{table}

These mappings are underwriting primitives, not mathematical decoration. The probability mapping \(Q^e\) turns operational exposure into expected event frequency; the severity mapping \(X^e\) turns the same exposure into a conditional loss magnitude; the permission-exposure map \(\psi_\eta\) turns granted tool rights into a weighted attack-surface measure; the governance-cost mapping \(K_i\) is the insured's cost of meeting a tier; the premium and deductible schedules \(\tau_i\) and \(d_i^r\) encode the insurer's governance credits; and the risk-loading mapping \(\varrho_i\) absorbs capital charges, dependency accumulation, and model and legal uncertainty.

\begin{proposition}[Risk-state sufficiency for underwriting maps]
\label{prop:risk_state_sufficiency}
Fix the underwriting maps \(Q^e:\mathcal S\to[0,1]\), \(X^e:\mathcal S\to\mathbb R_+\), \(\psi_\eta:\mathcal P_m\to\mathbb R_+\), the concentration statistic \(R:\mathcal V\to\mathbb R_+\), and the governance assignment rule \(h:[0,1]^{L_g}\to\mathcal G\). If two deployments have the same risk state \(s_i=s_j\), then they receive the same state-dependent event probabilities, severities, permission exposures, dependency-concentration scores, and governance-tier evaluations under these maps. Any remaining differences in premiums, costs, or risk loadings must therefore arise from insured-specific contract terms, cost functions, capital charges, or legal assumptions rather than from the operational risk state itself.
\end{proposition}

\begin{proof}
The conclusion follows by direct evaluation of the fixed maps at the common state. If \(s_i=s_j\), then \(Q^e(s_i)=Q^e(s_j)\) and \(X^e(s_i)=X^e(s_j)\) for every event \(e\in\mathcal E\). The equality of the permission, governance, and dependency coordinates implies \(\psi_\eta(\eta_i)=\psi_\eta(\eta_j)\), \(h(\mathbf z_i)=h(\mathbf z_j)\) whenever the same control scores induce the common governance tier, and \(R(v_i)=R(v_j)\). Quantities with insured-specific subscripts, such as \(K_i\), \(\tau_i\), or \(\varrho_i\), may still differ if the insurer assigns different cost, market, capital, or legal parameters.
\end{proof}

The symbols \(\rho_\alpha^e\), \(\xi_\alpha^e\), \(\kappa^e\), and \(\chi^e\) are not additional state variables. They are event-specific calibration maps used to translate ordinal categories into numerical modifiers inside the probability and severity examples below. The subscript \(\alpha\) indicates an autonomy-category effect; the governance maps do not carry this subscript because their input is \(g_i\). The superscript \(e\) indicates that different loss events may respond differently to the same autonomy category or governance tier. Table~\ref{tab:ordinal_effect_maps} summarizes their roles.

\begin{table}[ht]
\centering
\caption{Ordinal Effect Maps Used in the Example Risk Models}
\label{tab:ordinal_effect_maps}
\begin{tabularx}{\textwidth}{p{0.18\textwidth}p{0.23\textwidth}>{\raggedright\arraybackslash}X}
\hline
Map & Domain and range & Interpretation \\
\hline
\(\rho_\alpha^e\) & \(\mathcal A\to\mathbb R_+\) & Autonomy effect on event probability. It enters the logit of \(Q^e\); larger values mean that event \(e\) is more likely under the corresponding autonomy category. \\
\(\xi_\alpha^e\) & \(\mathcal A\to\mathbb R_+\) & Autonomy effect on conditional severity. It enters \(X^e\); larger values mean that event \(e\) tends to produce larger losses under the corresponding autonomy category. \\
\(\kappa^e\) & \(\mathcal G\to\mathbb R_+\) & Governance mitigation for event probability. It is subtracted in the logit of \(Q^e\); larger values mean that stronger governance lowers the probability of event \(e\). \\
\(\chi^e\) & \(\mathcal G\to\mathbb R_+\) & Governance mitigation for conditional severity. It enters through \(\exp(-\chi^e(g_i))\); larger values mean that stronger governance lowers the loss severity of event \(e\). \\
\hline
\end{tabularx}
\end{table}

These maps exist because autonomy categories and governance tiers are ordinal: they let an insurer specify monotone category effects without pretending the gap between adjacent categories is numerically equal. In practice they are calibrated from underwriting scorecards, expert elicitation, red-team results, control audits, incident and cyber claims data, stress scenarios, and portfolio accumulation studies. The next two equations are therefore transparent benchmark specifications, not the only admissible functional forms; they convert the state variables into event probabilities and conditional severities while preserving the signs and monotonicities an underwriter would expect.

For event probabilities, a logistic mapping is natural because the output must remain between a lower reference frequency and an upper stress frequency. The inner index can be interpreted as a latent risk score for event \(e\): autonomy, delegated authority, permission exposure, and dependency concentration increase this score, while stronger governance lowers it. The lower bound \(\bar q^e\) captures residual event frequency even for well-controlled deployments, and \(q_{\max}^e\) captures the largest annual frequency that the insurer regards as plausible for the class of event under stress. One illustrative event-probability mapping is
\begin{equation}
Q^e(s_i)
=
\bar q^e
+
\left(q_{\max}^e-\bar q^e\right)
\sigma\!\left(
a_0^e+\rho_\alpha^e(\alpha_i)+a_\beta^e\beta_i
+a_\eta^e\psi_\eta(\eta_i)
-\kappa^e(g_i)
+a_R^eR(v_i)
\right),
\label{eq:example_event_probability}
\end{equation}
where \(\sigma(z)=(1+e^{-z})^{-1}\), \(0\le \bar q^e<q_{\max}^e\le1\), and the coefficients are event specific. The autonomy-risk map \(\rho_\alpha^e:\mathcal A\to\mathbb R_+\) and the governance-mitigation map \(\kappa^e:\mathcal G\to\mathbb R_+\) satisfy \(\rho_\alpha^e(\alpha^{(0)})\le\cdots\le\rho_\alpha^e(\alpha^{(4)})\) and \(\kappa^e(g^{(1)})\le\cdots\le\kappa^e(g^{(M)})\). This formulation respects the ordinal nature of both autonomy categories and governance tiers: it does not treat adjacent categories as equally spaced numerical levels.

For conditional severity, the modeling requirement is different. The loss amount \(X^e(s_i)\) must remain nonnegative, should grow with operational exposure, and should allow governance controls to reduce expected loss size without making losses negative. The following multiplicative specification serves that purpose. The baseline severity \(b_0^e\) represents the conditional loss for a low-exposure deployment, the first factor increases severity with autonomy, delegated authority, and weighted permissions, the second factor captures dependency concentration, and the exponential governance term applies a proportional mitigation credit. A compatible severity mapping is
\begin{equation}
X^e(s_i)
=
b_0^e
\left(1+\xi_\alpha^e(\alpha_i)+b_\beta^e\beta_i+b_\eta^e\psi_\eta(\eta_i)\right)
\left(1+b_R^eR(v_i)\right)
\exp\!\left(-\chi^e(g_i)\right),
\label{eq:example_severity}
\end{equation}
where \(\xi_\alpha^e:\mathcal A\to\mathbb R_+\) and \(\chi^e:\mathcal G\to\mathbb R_+\) are nondecreasing in the category order, so \(\xi_\alpha^e(\alpha^{(0)})\le\cdots\le\xi_\alpha^e(\alpha^{(4)})\) and \(\chi^e(g^{(1)})\le\cdots\le\chi^e(g^{(M)})\). Higher operational authority, broader weighted permission exposure, and stronger dependency concentration increase severity, while higher governance tiers reduce it through the factor \(\exp(-\chi^e(g_i))\). Governance cost is also a discrete mapping \(K_i:\mathcal G\to\mathbb R_+\), with \(K_i(g^{(1)})\le K_i(g^{(2)})\le\cdots\le K_i(g^{(M)})\), reflecting the higher cost of stronger monitoring, validation, logging, approval workflows, and incident-response readiness.

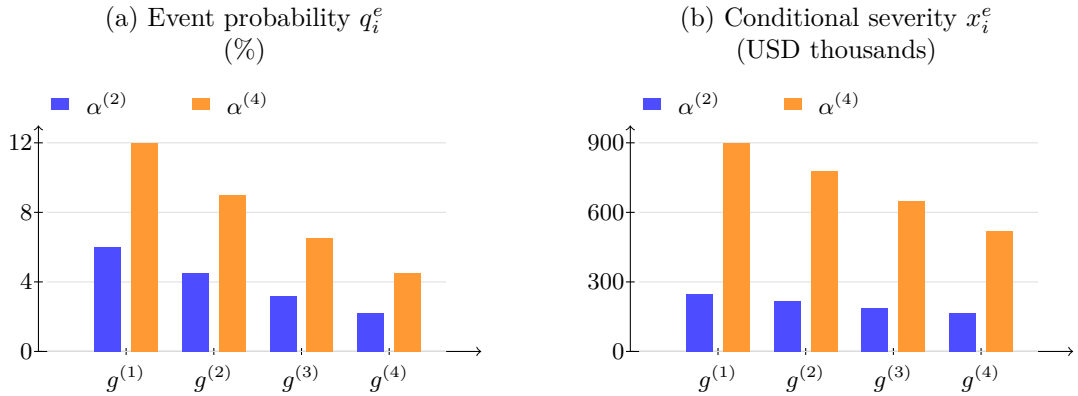
\begin{figure}[!htbp]
\centering
\begin{tikzpicture}[font=\footnotesize,x=1.16cm,y=0.92cm]
\begin{scope}[shift={(0,0)}]
\draw[->] (0,0) -- (5.05,0);
\draw[->] (0,0) -- (0,3.25);
\foreach \x/\lab in {1/{\(g^{(1)}\)},2/{\(g^{(2)}\)},3/{\(g^{(3)}\)},4/{\(g^{(4)}\)}} {
  \draw (\x,0.05) -- (\x,-0.05) node[below] {\lab};
}
\foreach \y/\lab in {0/0,1/4,2/8,3/12} {
  \draw (-0.05,\y) -- (0.05,\y) node[left] {\lab};
  \draw[thin,gray!25] (0.10,\y) -- (4.65,\y);
}
\fill[blue!70] (0.64,0) rectangle (0.94,1.50);
\fill[blue!70] (1.64,0) rectangle (1.94,1.13);
\fill[blue!70] (2.64,0) rectangle (2.94,0.80);
\fill[blue!70] (3.64,0) rectangle (3.94,0.55);
\fill[orange!80] (1.06,0) rectangle (1.36,3.00);
\fill[orange!80] (2.06,0) rectangle (2.36,2.25);
\fill[orange!80] (3.06,0) rectangle (3.36,1.63);
\fill[orange!80] (4.06,0) rectangle (4.36,1.13);
\node[align=center,font=\small] at (2.35,4.55) {(a) Event probability \(q_i^e\)\\(\%)};
\fill[blue!70] (0.15,3.48) rectangle (0.35,3.65);
\node[anchor=west] at (0.43,3.57) {\(\alpha^{(2)}\)};
\fill[orange!80] (1.75,3.48) rectangle (1.95,3.65);
\node[anchor=west] at (2.03,3.57) {\(\alpha^{(4)}\)};
\end{scope}

\begin{scope}[shift={(6.75,0)}]
\draw[->] (0,0) -- (5.05,0);
\draw[->] (0,0) -- (0,3.25);
\foreach \x/\lab in {1/{\(g^{(1)}\)},2/{\(g^{(2)}\)},3/{\(g^{(3)}\)},4/{\(g^{(4)}\)}} {
  \draw (\x,0.05) -- (\x,-0.05) node[below] {\lab};
}
\foreach \y/\lab in {0/0,1/300,2/600,3/900} {
  \draw (-0.05,\y) -- (0.05,\y) node[left] {\lab};
  \draw[thin,gray!25] (0.10,\y) -- (4.65,\y);
}
\fill[blue!70] (0.64,0) rectangle (0.94,0.83);
\fill[blue!70] (1.64,0) rectangle (1.94,0.73);
\fill[blue!70] (2.64,0) rectangle (2.94,0.63);
\fill[blue!70] (3.64,0) rectangle (3.94,0.55);
\fill[orange!80] (1.06,0) rectangle (1.36,3.00);
\fill[orange!80] (2.06,0) rectangle (2.36,2.60);
\fill[orange!80] (3.06,0) rectangle (3.36,2.17);
\fill[orange!80] (4.06,0) rectangle (4.36,1.73);
\node[align=center,font=\small] at (2.35,4.55) {(b) Conditional severity \(x_i^e\)\\(USD thousands)};
\fill[blue!70] (0.15,3.48) rectangle (0.35,3.65);
\node[anchor=west] at (0.43,3.57) {\(\alpha^{(2)}\)};
\fill[orange!80] (1.75,3.48) rectangle (1.95,3.65);
\node[anchor=west] at (2.03,3.57) {\(\alpha^{(4)}\)};
\end{scope}
\end{tikzpicture}
\caption{Illustrative bar-plot examples of risk mappings evaluated at discrete governance tiers. Panel (a) shows annual event probability and panel (b) shows conditional severity for two autonomy categories. The values are order-of-magnitude underwriting examples for a healthcare-style agentic-AI deployment, not empirical estimates.}
\label{fig:governance_tier_functions}
\end{figure}

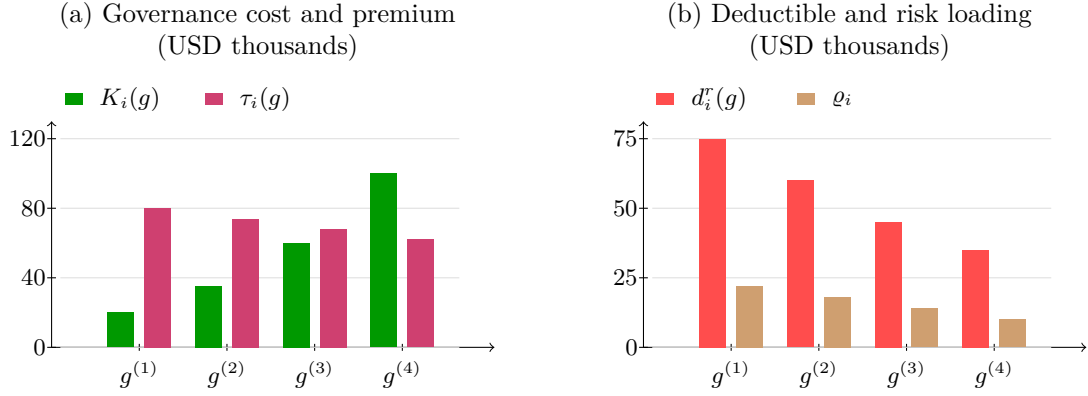
\begin{figure}[!htbp]
\centering
\begin{tikzpicture}[font=\footnotesize,x=1.16cm,y=0.92cm]
\begin{scope}[shift={(0,0)}]
\draw[->] (0,0) -- (5.05,0);
\draw[->] (0,0) -- (0,3.25);
\foreach \x/\lab in {1/{\(g^{(1)}\)},2/{\(g^{(2)}\)},3/{\(g^{(3)}\)},4/{\(g^{(4)}\)}} {
  \draw (\x,0.05) -- (\x,-0.05) node[below] {\lab};
}
\foreach \y/\lab in {0/0,1/40,2/80,3/120} {
  \draw (-0.05,\y) -- (0.05,\y) node[left] {\lab};
  \draw[thin,gray!25] (0.10,\y) -- (4.65,\y);
}
\fill[green!60!black] (0.64,0) rectangle (0.94,0.50);
\fill[green!60!black] (1.64,0) rectangle (1.94,0.88);
\fill[green!60!black] (2.64,0) rectangle (2.94,1.50);
\fill[green!60!black] (3.64,0) rectangle (3.94,2.50);
\fill[purple!75] (1.06,0) rectangle (1.36,2.00);
\fill[purple!75] (2.06,0) rectangle (2.36,1.85);
\fill[purple!75] (3.06,0) rectangle (3.36,1.70);
\fill[purple!75] (4.06,0) rectangle (4.36,1.55);
\node[align=center,font=\small] at (2.35,4.55) {(a) Governance cost and premium\\(USD thousands)};
\fill[green!60!black] (0.15,3.48) rectangle (0.35,3.65);
\node[anchor=west] at (0.43,3.57) {\(K_i(g)\)};
\fill[purple!75] (1.75,3.48) rectangle (1.95,3.65);
\node[anchor=west] at (2.03,3.57) {\(\tau_i(g)\)};
\end{scope}

\begin{scope}[shift={(6.75,0)}]
\draw[->] (0,0) -- (5.05,0);
\draw[->] (0,0) -- (0,3.25);
\foreach \x/\lab in {1/{\(g^{(1)}\)},2/{\(g^{(2)}\)},3/{\(g^{(3)}\)},4/{\(g^{(4)}\)}} {
  \draw (\x,0.05) -- (\x,-0.05) node[below] {\lab};
}
\foreach \y/\lab in {0/0,1/25,2/50,3/75} {
  \draw (-0.05,\y) -- (0.05,\y) node[left] {\lab};
  \draw[thin,gray!25] (0.10,\y) -- (4.65,\y);
}
\fill[red!70] (0.64,0) rectangle (0.94,3.00);
\fill[red!70] (1.64,0) rectangle (1.94,2.40);
\fill[red!70] (2.64,0) rectangle (2.94,1.80);
\fill[red!70] (3.64,0) rectangle (3.94,1.40);
\fill[brown!75] (1.06,0) rectangle (1.36,0.88);
\fill[brown!75] (2.06,0) rectangle (2.36,0.72);
\fill[brown!75] (3.06,0) rectangle (3.36,0.56);
\fill[brown!75] (4.06,0) rectangle (4.36,0.40);
\node[align=center,font=\small] at (2.35,4.55) {(b) Deductible and risk loading\\(USD thousands)};
\fill[red!70] (0.15,3.48) rectangle (0.35,3.65);
\node[anchor=west] at (0.43,3.57) {\(d_i^r(g)\)};
\fill[brown!75] (1.75,3.48) rectangle (1.95,3.65);
\node[anchor=west] at (2.03,3.57) {\(\varrho_i\)};
\end{scope}
\end{tikzpicture}
\caption{Illustrative bar-plot examples of financial and contract-schedule mappings evaluated at discrete governance tiers. Panel (a) shows governance cost and premium schedules; panel (b) shows deductible and risk-loading schedules. The values are order-of-magnitude underwriting examples for a healthcare-style agentic-AI deployment, not empirical estimates.}
\label{fig:financial_schedule_functions}
\end{figure}

\subsection{Agentic-AI Event Space}

Traditional cyber-insurance models classify losses by attack category---phishing, ransomware, denial of service, malware, or insider threats. Agentic AI needs a broader event space, because losses may arise from autonomous behavior, flawed reasoning, degraded model performance, dependency failures, or cyber-physical interaction even when no conventional intrusion occurs.

Let \(\mathcal R=\{\mathrm{Cyber},\mathrm{TechEO},\mathrm{Perf},\mathrm{GL},\mathrm{Mixed}\}\) denote the set of coverage layers. For each layer \(r\in\mathcal R\), let \(\mathcal E^r\) denote the set of event types naturally associated with that layer, and let \(\mathcal E=\bigcup_{r\in\mathcal R}\mathcal E^r\) denote the full event space. The sets \(\mathcal E^r\) need not be disjoint because a single event may implicate several coverage layers.

\noindent Table~\ref{tab:event_taxonomy} lists representative events for each coverage layer. The symbols \(e_H\), \(e_P\), \(e_F\), \(e_D\), \(e_O\), and \(e_C\) are used later in the healthcare case study as a reduced event set.

\begin{table}[ht]
\centering
\caption{Representative Agentic-AI Events by Coverage Layer}
\label{tab:event_taxonomy}
\begin{tabular}{p{0.18\textwidth}p{0.72\textwidth}}
\hline
Layer \(r\) & Representative events \(e\in\mathcal E^r\) \\
\hline
Cyber & Prompt injection \(e_P\), data exfiltration, credential misuse, unauthorized tool invocation, malicious configuration change \\
TechE\&O & Hallucinated advice \(e_H\), faulty output, negligent automation, workflow misrouting, erroneous customer or patient communication \\
Performance & Model drift \(e_D\), calibration failure, benchmark degradation, latency failure, reliability degradation \\
GL & Cyber-physical harm \(e_C\), bodily injury, property damage, equipment malfunction, unsafe robotic or IoT control \\
Mixed & Agentic fraud \(e_F\), dependency outage \(e_O\), multi-causal loss, unclear causation across cyber, service, and AI-behavioral factors \\
\hline
\end{tabular}
\end{table}

Hallucination events arise when an AI system produces incorrect outputs, recommendations, or actions that are relied upon by users or automated workflows. Prompt-injection events occur when adversarial inputs manipulate the reasoning process of the AI system, causing it to disclose information, ignore safeguards, or execute unauthorized actions \cite{perez2022ignore,greshake2023not}. Agentic fraud includes unauthorized transactions, deceptive communications, and abuse of operational authority. Model drift captures performance degradation under changing operating conditions \cite{quinonero2009dataset}. Dependency outages arise from failures of cloud providers, model vendors, external APIs, software connectors, or hosted agent platforms. Cyber-physical harm covers bodily injury, property damage, equipment malfunction, or critical-infrastructure disruption resulting from autonomous interaction with physical environments \cite{humayed2017cyber}.

\subsection{Coverage Architecture}

A defining feature of agentic-AI insurance is that a single event may simultaneously involve autonomous decision making, software malfunction, cyber compromise, professional negligence, regulatory violation, and physical consequence. The central modeling problem is thus not only whether coverage exists, but how an event maps into the appropriate coverage layers when several causal mechanisms contribute to the loss.

For insured \(i\), an agentic-AI insurance contract is
\begin{equation}
C_i
=
\left(
T_i,\{D_i^r,L_i^r\}_{r\in\mathcal R},A_i,\Gamma_i,\Lambda_i,\Psi_i
\right),
\label{eq:contract_definition}
\end{equation}
where \(T_i\in\mathbb R_+\) is the premium, \(D_i^r\in\mathbb R_+\) is the deductible for layer \(r\), \(L_i^r\in\mathbb R_+\) is the per-event layer limit, \(A_i\in\mathbb R_+\) is the AI aggregate limit, and \(\Psi_i\) is the governance-obligation component of the policy. The binary coverage-incidence matrix \(\Gamma_i\in\{0,1\}^{|\mathcal E|\times|\mathcal R|}\) has entries \(\gamma_i^{e,r}\), and the indemnity-share matrix \(\Lambda_i\in[0,1]^{|\mathcal E|\times|\mathcal R|}\) has entries \(\lambda_i^{e,r}\).

\noindent The binary variable \(\gamma_i^{e,r}\) is equal to one when event \(e\) is covered or assigned under layer \(r\) for insured \(i\), and is zero otherwise. Thus \(\Gamma_i\) is a matrix of zeros and ones, not a fractional allocation matrix. The continuous share \(\lambda_i^{e,r}\) determines the fraction of the payable loss attributed to layer \(r\), and it is admissible only when the corresponding incidence entry is active:
\begin{equation}
0\le \lambda_i^{e,r}\le \gamma_i^{e,r},
\qquad
\sum_{r\in\mathcal R}\lambda_i^{e,r}\le 1,
\qquad
\forall e\in\mathcal E,\;r\in\mathcal R.
\label{eq:allocation_constraint}
\end{equation}
\noindent This separation is useful in mixed-cause events. The matrix \(\Gamma_i\) records which coverage layers are legally or contractually available, while \(\Lambda_i\) records how the indemnity is apportioned among those available layers for pricing and claims execution.

\subsection{Governance Covenants and Policy Obligations}
\label{subsec:governance_covenants}

The term \(\Psi_i\) is the governance-obligation component of the policy---the part of the contract that converts underwriting assumptions into continuing obligations on the insured deployment. This matters because an agentic-AI risk state is not fixed at inception: delegated authority, permission scope, model dependencies, and operational capabilities may all shift during the policy period. \(\Psi_i\) should therefore be read as enforceable policy language in mathematical form, not merely a vector of model parameters.

Formally, we write \(\Psi_i=(g_i^\star,\bar{\mathbf z}_i,\mathcal O_i,\mathcal M_i,\mathcal N_i)\). Here \(g_i^\star\in\mathcal G\) is the required governance tier, \(\bar{\mathbf z}_i=(\bar z_{i,1},\ldots,\bar z_{i,L_g})\) gives minimum acceptable control scores, \(\mathcal O_i\) contains oversight obligations, \(\mathcal M_i\) contains monitoring and evidence obligations, and \(\mathcal N_i\) contains notice and remediation obligations. The map \(g_{\min}(\Psi_i)\) returns the minimum governance tier capable of satisfying the covenant package.

A representative policy covenant may be written as follows.
\begin{quote}
\small
\textit{Agentic-AI governance covenant.} The insured shall maintain human approval for safety-critical actions; tamper-evident audit logging of AI outputs, tool calls, approvals, and system actions for at least twenty-four months; quarterly prompt-injection and adversarial testing; written notice to the insurer within thirty days after any material expansion of AI permissions, autonomous authority, or external-system access; and incident notice within seventy-two hours after discovery of any event reasonably expected to give rise to a claim. Failure to maintain these controls may result in loss of governance-related premium credits, increased deductibles, suspension of coverage for affected AI-related claims, or other remedies permitted by the policy and applicable law.
\end{quote}

Each clause has a direct interpretation in the model. Oversight obligations in \(\mathcal O_i\) restrict operational authority. If safety-critical actions must be approved by a human, then those actions have \(h_i(u)=1\) in \eqref{eq:human_approval_indicator}. More generally, the covenant may impose an upper bound \(\beta_i\le\bar\beta_i\), where \(\bar\beta_i\) is determined during underwriting. This prevents the insured from expanding autonomous execution authority after coverage is priced.

Monitoring and evidence obligations in \(\mathcal M_i\) translate the threshold vector \(\bar{\mathbf z}_i\) into auditable requirements. For a logging, testing, access-control, or incident-response dimension \(\ell\), the policy requirement is represented by \(z_{i,\ell}\ge\bar z_{i,\ell}\). Since the assigned governance tier is \(g_i=h(\mathbf z_i)\), these thresholds determine whether the insured continues to satisfy the required governance tier \(g_i^\star\). Audit logs, retained approval records, telemetry, prompt-injection tests, and model-performance reports therefore serve as evidence for the control scores used by the underwriting rule.

The notice obligation for permission expansion is tied to the permission profile \(\eta_i\). A deployment approved only for patient messaging and scheduling may later be given authority to modify patient records, initiate payments, invoke external APIs, or control devices. Such changes alter \(\eta_i\), the weighted exposure \(\psi_\eta(\eta_i)\), and potentially the admissible action set \(\mathcal U_i(\eta_i)\). The covenant therefore requires notice and possible re-underwriting before material permission expansions become part of the insured operating environment.

Notice and remediation duties in \(\mathcal N_i\) govern post-incident claims handling. They apply after a material incident, control failure, or potentially covered AI-related loss. Prompt notice preserves logs and causal evidence, allows the insurer to evaluate whether the event falls within the covered event set \(\mathcal E\), and supports the indemnity calculation through \(\Gamma_i\), \(\Lambda_i\), and the layer-payment functions. Remediation duties, such as an incident report within fourteen days and a remediation plan within thirty days, help prevent repeated losses from the same control failure.

For the covenant quoted above, a concrete mathematical representation could be written as
\begin{equation}
\label{eq:example_covenant_encoding}
\begin{aligned}
\Psi_i^{\mathrm{cov}}
&=
\left(g^{(3)},\bar{\mathbf z}_i,
\mathcal O_i^{\mathrm{cov}},
\mathcal M_i^{\mathrm{cov}},
\mathcal N_i^{\mathrm{cov}}\right),
\\
\mathcal O_i^{\mathrm{cov}}
&:\quad
h_i(u)=1\ \ \forall u\in\mathcal U_i^{\mathrm{crit}},
\qquad
\beta_i\le\bar\beta_i,
\\
\mathcal M_i^{\mathrm{cov}}
&:\quad
z_{i,\ell_{\mathrm{log}}}\ge\bar z_{i,\ell_{\mathrm{log}}},
\qquad
z_{i,\ell_{\mathrm{test}}}\ge\bar z_{i,\ell_{\mathrm{test}}},
\\
&\quad
R_i^{\mathrm{log}}\ge24,
\qquad
\Delta_i^{\mathrm{test}}\le90,
\\
\mathcal N_i^{\mathrm{cov}}
&:\quad
t_i^\eta\le30,
\qquad
t_i^e\le72.
\end{aligned}
\end{equation}
Here \(\mathcal U_i^{\mathrm{crit}}\) is the set of safety-critical actions, \(R_i^{\mathrm{log}}\) is the number of months for which tamper-evident logs are retained, \(\Delta_i^{\mathrm{test}}\) is the maximum number of days between prompt-injection or adversarial tests, \(t_i^\eta\) is the number of days before notice is provided after a material permission or authority expansion, and \(t_i^e\) is the number of hours before incident notice is provided after discovery of a potentially covered event. A material permission expansion can be operationalized by comparing the post-change permission vector \(\eta_i^+\) with the underwritten vector \(\eta_i\), for example by requiring re-underwriting whenever \(\psi_\eta(\eta_i^+)-\psi_\eta(\eta_i)\ge\varepsilon_\eta\) for an insurer-chosen threshold \(\varepsilon_\eta\), or whenever a scheduled high-risk permission class is newly activated. This example shows how the legal language of the covenant becomes constraints on human approval, operational authority, control evidence, testing frequency, permission changes, and notice timing.

The covenant package connects directly to the optimization framework. It imposes the feasibility condition \(g_i\succeq g_{\min}(\Psi_i)\), shifts event probabilities and severities through \(q_i^e=Q^e(s_i)\) and \(x_i^e=X^e(s_i)\), and supplies the contractual mechanism behind governance-sensitive premium and deductible schedules. Premium credits, deductible credits, coverage enhancements, or continued eligibility may be conditioned on maintaining the required tier, while breach of \(\Psi_i\) may trigger repricing, higher deductibles, withdrawal of credits, non-renewal, or other policy remedies. Governance covenants are thus not administrative details; they are how agentic-AI controls become insurable, monitorable, and enforceable across the policy lifecycle.

\subsection{Indemnity Function}

For event \(e\), the layer-level indemnity attributed to coverage layer \(r\) is
\begin{equation}
Y_i^{e,r}(C_i)
=
\lambda_i^{e,r}
\min\left\{(x_i^e-D_i^r)^+,L_i^r\right\},
\qquad
(x)^+=\max\{x,0\}.
\label{eq:layer_indemnity}
\end{equation}
\noindent Let \(\widehat Y_i^e(C_i)=\sum_{r\in\mathcal R}Y_i^{e,r}(C_i)\) be the event payment before the annual aggregate and let \(\widehat Y_i(C_i)=\sum_{e\in\mathcal E}\widehat Y_i^e(C_i)\). The aggregate-adjusted event indemnity is
\begin{equation}
Y_i^e(C_i)
=
\begin{cases}
\widehat Y_i^e(C_i)\min\{1,A_i/\widehat Y_i(C_i)\}, & \widehat Y_i(C_i)>0,\\
0, & \widehat Y_i(C_i)=0.
\end{cases}
\label{eq:indemnity}
\end{equation}
\noindent Equations \eqref{eq:layer_indemnity}--\eqref{eq:indemnity} preserve the standard roles of deductibles, limits, and aggregate caps while enabling causation-sensitive allocation across the cyber, technology E\&O, performance, general-liability, and mixed-cause layers. Losses below \(D_i^r\) stay with the insured, payments above the deductible are capped by \(L_i^r\), and \(A_i\) bounds the insurer's AI-related aggregate exposure over the covered event set.

Figure~\ref{fig:layer_indemnity_schedule} illustrates the layer-level payment schedule in \eqref{eq:layer_indemnity} for a fixed event \(e\) and coverage layer \(r\). The deductible creates a no-payment region for small losses, the layer limit creates a capped payment region for large losses, and the allocation share \(\lambda_i^{e,r}\) scales the amount paid by layer \(r\) when the event is allocated across several layers.

\begin{figure}[!htbp]
\centering
\begin{tikzpicture}[font=\footnotesize,x=1.45cm,y=1.25cm]
\draw[->] (0,0) -- (6.55,0) node[right] {\(x_i^e\) (USD thousands)};
\draw[->] (0,0) -- (0,3.50) node[above] {\(Y_i^{e,r}(C_i)\) (USD thousands)};
\foreach \x/\lab in {0/0,1/100,2/200,3/300,4/400,5/500,6/600} {
  \draw (\x,0.05) -- (\x,-0.05) node[below] {\lab};
}
\foreach \y/\lab in {0/0,1/100,2/200,3/300} {
  \draw (-0.05,\y) -- (0.05,\y) node[left] {\lab};
  \draw[thin,gray!25] (0.08,\y) -- (6.20,\y);
}
\draw[dashed,gray!70] (0.50,0) -- (0.50,3.05);
\draw[dashed,gray!70] (3.00,0) -- (3.00,3.05);
\draw[dashed,gray!70] (0,2.50) -- (6.20,2.50);
\draw[dashed,gray!70] (0,1.25) -- (6.20,1.25);
\node[below] at (0.50,-0.25) {\(D_i^r\)};
\node[below] at (3.00,-0.25) {\(D_i^r+L_i^r\)};
\node[left] at (-0.02,2.50) {\(L_i^r\)};
\node[left] at (-0.02,1.25) {\(0.5L_i^r\)};
\draw[very thick,blue!70] (0,0) -- (0.50,0) -- (3.00,2.50) -- (6.20,2.50);
\draw[very thick,orange!85,dashed] (0,0) -- (0.50,0) -- (3.00,1.25) -- (6.20,1.25);
\fill[blue!70] (3.75,3.25) rectangle (3.95,3.40);
\node[anchor=west] at (4.02,3.32) {\(\lambda_i^{e,r}=1\)};
\draw[very thick,orange!85,dashed] (3.75,3.02) -- (3.95,3.02);
\node[anchor=west] at (4.02,3.02) {\(\lambda_i^{e,r}=0.5\)};
\node[anchor=west,fill=white,fill opacity=0.96,text opacity=1,inner sep=2pt] at (0.82,0.48) {deductible region};
\draw[->,gray!70] (0.80,0.38) -- (0.28,0.08);
\node[anchor=west,fill=white,fill opacity=0.96,text opacity=1,inner sep=2pt] at (0.95,3.08) {capped payment region};
\draw[->,gray!70] (2.60,2.96) -- (4.25,2.52);
\end{tikzpicture}
\caption{Illustrative layer-level indemnity schedule \(Y_i^{e,r}(C_i)=\lambda_i^{e,r}\min\{(x_i^e-D_i^r)^+,L_i^r\}\) as a function of gross loss \(x_i^e\). The example uses \(D_i^r=\$50{,}000\) and \(L_i^r=\$250{,}000\). The solid curve shows full allocation to layer \(r\), while the dashed curve shows a mixed-cause allocation in which only half of the payable loss is attributed to that layer.}
\label{fig:layer_indemnity_schedule}
\end{figure}
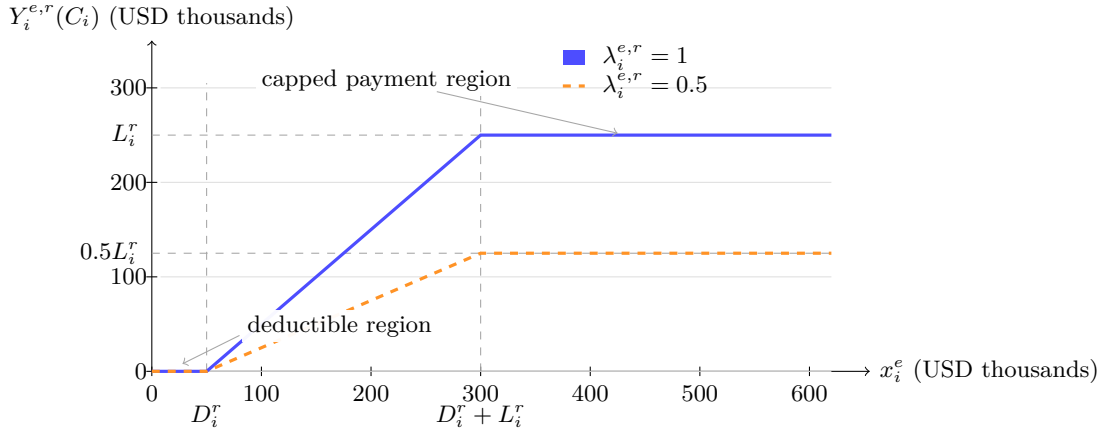

\section{Pricing and Contract Optimization Framework}
\label{sec:contract_design}

The preceding section defined the risk state, event taxonomy, coverage layers, and indemnity function. This section uses those objects to state the pricing, participation, incentive-compatibility, and optimization conditions that determine a feasible agentic-AI insurance contract. Throughout, \(T_i\), \(D_i^r\), \(L_i^r\), \(A_i\), \(\Gamma_i\), and \(\Lambda_i\) are the scalar or matrix decision variables of the contract. When premium or deductible schedules create governance incentives, they are written as separate functions \(\tau_i(\cdot)\) and \(d_i^r(\cdot)\), and the issued contract carries the realized scalars \(T_i=\tau_i(g_i^\star)\) and \(D_i^r=d_i^r(g_i^\star)\), where \(g_i^\star\in\mathcal G\) is the required governance tier.

\subsection{Insured Utility and Participation}

Insurance is economically meaningful only if it improves the insured's expected financial position. Let \(K_i(g_i)\) denote the annual cost of the governance controls associated with tier \(g_i\): monitoring, logging, approval workflows, red-teaming, model validation, rollback capability, and incident response. To keep notation clean, we separate the participation cost of an issued contract from the counterfactual governance-deviation cost used later for incentive compatibility. The insured's participation cost under the issued contract \(C_i\) is
\begin{equation}
J_i^{\mathrm{part}}(C_i)
=
T_i
+
K_i(g_i)
+
\sum_{e\in\mathcal E}
q_i^e
\left(
x_i^e-Y_i^e(C_i)
\right).
\label{eq:insured_cost}
\end{equation}
\noindent This expression evaluates the realized contract terms \(T_i\), \(D_i^r\), \(L_i^r\), \(A_i\), \(\Gamma_i\), and \(\Lambda_i\) at the realized governance tier \(g_i\). It is the object used only for the participation comparison against the uninsured baseline.
\noindent The baseline cost without insurance is
\begin{equation}
J_i^0
=
K_i(g_i^0)
+
\sum_{e\in\mathcal E}
q_i^{e,0}
x_i^{e,0},
\label{eq:baseline_cost}
\end{equation}
\noindent where \(g_i^0\), \(q_i^{e,0}\), and \(x_i^{e,0}\) denote the governance tier, event probability, and gross loss under the uninsured operating policy. Individual rationality requires \(J_i^{\mathrm{part}}(C_i)\le J_i^0\), which is the insured's participation constraint \cite{rothschild1976equilibrium}.

\subsection{Insurer Profitability and Risk-Based Pricing}

The insurer's expected payout for insured \(i\) is \(\sum_{e\in\mathcal E}q_i^eY_i^e(C_i)\). Because agentic-AI insurance faces sparse loss data, model evolution, legal uncertainty, and correlated dependency risk, expected-loss pricing must be supplemented by the risk-loading function \(\varrho_i:\mathcal S\times\mathcal C_i\to\mathbb R_+\) \cite{bohme2010modeling,biener2015insurability}. The risk-loaded premium condition is
\begin{equation}
T_i
\ge
\sum_{e\in\mathcal E}
q_i^eY_i^e(C_i)
+
\varrho_i(s_i,C_i).
\label{eq:risk_loaded_premium}
\end{equation}
\noindent The loading \(\varrho_i(s_i,C_i)\) should be interpreted as a total underwriting adjustment rather than as a required decomposition into named subcomponents. In practice, the insurer may set this loading to reflect portfolio accumulation, dependency concentration, model uncertainty, sparse claims experience, legal uncertainty, regulatory exposure, and the amount of limit deployed. Figure~\ref{fig:risk_loading_examples} gives an illustrative bar-plot example of how the total loading may increase as underwriting uncertainty and concentration risk increase.

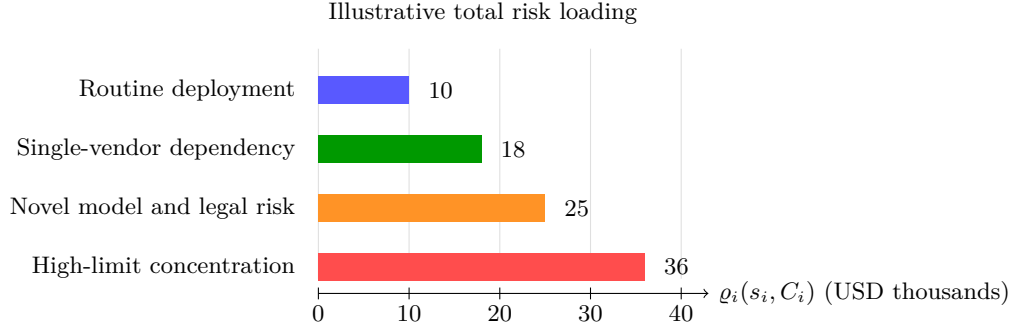
\begin{figure}[!htbp]
\centering
\begin{tikzpicture}[font=\footnotesize,x=0.12cm,y=0.78cm]
\node[anchor=west] at (0,5.30) {Illustrative total risk loading};
\draw[->] (0,0.55) -- (43,0.55) node[right] {\(\varrho_i(s_i,C_i)\) (USD thousands)};
\foreach \x/\lab in {0/0,10/10,20/20,30/30,40/40} {
  \draw (\x,0.45) -- (\x,0.65) node[below=3pt] {\lab};
  \draw[thin,gray!25] (\x,0.80) -- (\x,4.70);
}
\foreach \y/\lab/\val/\col in {
  4/{Routine deployment}/10/blue!65,
  3/{Single-vendor dependency}/18/green!60!black,
  2/{Novel model and legal risk}/25/orange!85,
  1/{High-limit concentration}/36/red!70
} {
  \node[anchor=east,align=right] at (-1.4,\y) {\lab};
  \fill[\col] (0,\y-0.23) rectangle (\val,\y+0.23);
  \node[anchor=west] at (\val+1.0,\y) {\val};
}
\end{tikzpicture}
\caption{Illustrative total risk-loading levels for four underwriting scenarios. The values are order-of-magnitude examples in thousands of dollars and are not empirical estimates. The point is that \(\varrho_i(s_i,C_i)\) can be treated as a single underwriting loading that increases with dependency concentration, model uncertainty, legal uncertainty, and deployed policy limits.}
\label{fig:risk_loading_examples}
\end{figure}
\noindent The insurer's risk-adjusted underwriting profit is
\begin{equation}
\Pi_i(C_i)
=
T_i
-
\sum_{e\in\mathcal E}
q_i^eY_i^e(C_i)
-
\varrho_i(s_i,C_i).
\label{eq:insurer_profit}
\end{equation}

\subsection{Governance Incentive Compatibility}

A distinguishing feature of agentic-AI insurance is that many important risk factors stay under the insured's direct control. After buying coverage, an insured may weaken monitoring, reduce human oversight, expand autonomous execution authority, or relax operational controls---raising expected losses through classic moral hazard \cite{liu2022moralhazard}. Governance must therefore enter the contract-design problem, consistent with the insurance-economics distinction among risk transfer, self-insurance, and self-protection \cite{ehrlich1972market}.

Let \(g_i^\star\in\mathcal G\) denote the governance tier required by the insurer and define \(s_i(\tilde g)=(\alpha_i,\beta_i,\eta_i,\tilde g,v_i)\). For a candidate governance tier \(\tilde g\in\mathcal G\), define the induced event probability and severity as \(q_i^e(\tilde g)=Q^e(s_i(\tilde g))\) and \(x_i^e(\tilde g)=X^e(s_i(\tilde g))\). The insured's governance-deviation cost is then
\begin{equation}
J_i^{\mathrm{gov}}(\tilde g;C_i)
=
\tau_i(\tilde g)
+
K_i(\tilde g)
+
\sum_{e\in\mathcal E}
q_i^e(\tilde g)
\left[
x_i^e(\tilde g)
-
Y_i^e(\tilde g;C_i)
\right].
\label{eq:insured_objective_governance}
\end{equation}
\noindent Here \(C_i\) is fixed while \(\tilde g\) varies; when governance-sensitive premium or deductible schedules are used, \(C_i\) is understood to include the schedules \(\tau_i(\cdot)\) and \(d_i^r(\cdot)\) for this counterfactual calculation. The participation cost satisfies \(J_i^{\mathrm{part}}(C_i)=J_i^{\mathrm{gov}}(g_i;C_i)\) when the issued contract is evaluated at its realized governance tier. The governance requirement is incentive compatible if \(g_i^\star\in\arg\min_{\tilde g\in\mathcal G}J_i^{\mathrm{gov}}(\tilde g;C_i)\), or equivalently if \(J_i^{\mathrm{gov}}(g_i^\star;C_i)\le J_i^{\mathrm{gov}}(g;C_i)\) for all \(g\in\mathcal G\). The insurer may induce this behavior through governance-sensitive schedules \(\tau_i:\mathcal G\to\mathbb R_+\) and \(d_i^r:\mathcal G\to\mathbb R_+\). Let \(T_i^{(m)}=\tau_i(g^{(m)})\) and \(D_i^{r,(m)}=d_i^r(g^{(m)})\). Premium credits and deductible credits for stronger governance can be imposed through \(T_i^{(1)}\ge T_i^{(2)}\ge\cdots\ge T_i^{(M)}\) and \(D_i^{r,(1)}\ge D_i^{r,(2)}\ge\cdots\ge D_i^{r,(M)}\). The scalar terms written into the issued contract are then \(T_i=\tau_i(g_i^\star)\) and \(D_i^r=d_i^r(g_i^\star)\). These schedules reward monitoring, audit logs, approval controls, validation procedures, and human oversight throughout the policy period without treating \(T_i\) or \(D_i^r\) themselves as functions. Figure~\ref{fig:governance_cost_objective} illustrates the counterfactual calculation: each bar is the total expected cost obtained by evaluating the fixed contract schedule at one discrete governance tier, so the induced tier is the bar with the lowest value rather than the minimizer of a continuous curve.

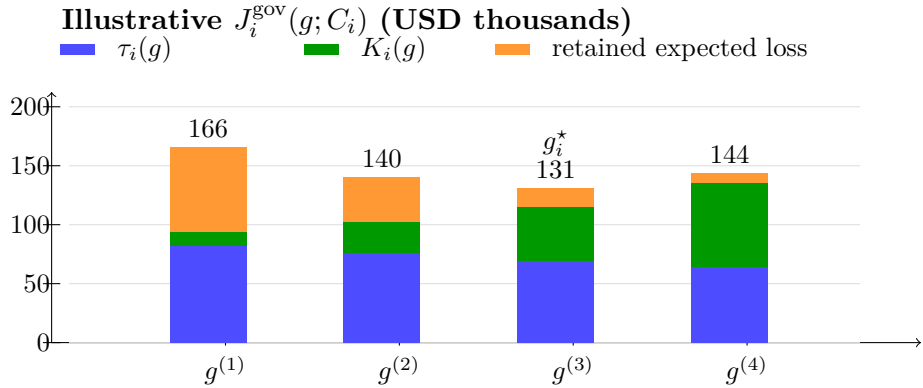
\begin{figure}[!htbp]
\centering
\begin{tikzpicture}[font=\small,x=2.30cm,y=0.78cm]
\draw[->] (0,0) -- (5.0,0);
\draw[->] (0,0) -- (0,4.25);
\foreach \y/\lab in {0/0,1/50,2/100,3/150,4/200} {
  \draw (-0.05,\y) -- (0.05,\y) node[left] {\lab};
  \draw[thin,gray!25] (0.10,\y) -- (4.65,\y);
}
\foreach \x/\lab in {1/{\(g^{(1)}\)},2/{\(g^{(2)}\)},3/{\(g^{(3)}\)},4/{\(g^{(4)}\)}} {
  \draw (\x,0.05) -- (\x,-0.05) node[below] {\lab};
}
\fill[blue!70] (0.68,0) rectangle (1.12,1.64);
\fill[green!60!black] (0.68,1.64) rectangle (1.12,1.88);
\fill[orange!80] (0.68,1.88) rectangle (1.12,3.32);
\node[above] at (0.90,3.32) {166};
\fill[blue!70] (1.68,0) rectangle (2.12,1.52);
\fill[green!60!black] (1.68,1.52) rectangle (2.12,2.04);
\fill[orange!80] (1.68,2.04) rectangle (2.12,2.80);
\node[above] at (1.90,2.80) {140};
\fill[blue!70] (2.68,0) rectangle (3.12,1.40);
\fill[green!60!black] (2.68,1.40) rectangle (3.12,2.30);
\fill[orange!80] (2.68,2.30) rectangle (3.12,2.62);
\node[above] at (2.90,2.62) {131};
\node[above] at (2.90,3.00) {\(g_i^\star\)};
\fill[blue!70] (3.68,0) rectangle (4.12,1.28);
\fill[green!60!black] (3.68,1.28) rectangle (4.12,2.72);
\fill[orange!80] (3.68,2.72) rectangle (4.12,2.88);
\node[above] at (3.90,2.88) {144};
\node[anchor=west,font=\bfseries] at (0.00,5.45) {Illustrative \(J_i^{\mathrm{gov}}(g;C_i)\) (USD thousands)};
\fill[blue!70] (0.05,4.88) rectangle (0.25,5.05);
\node[anchor=west] at (0.32,4.97) {\(\tau_i(g)\)};
\fill[green!60!black] (1.45,4.88) rectangle (1.65,5.05);
\node[anchor=west] at (1.72,4.97) {\(K_i(g)\)};
\fill[orange!80] (2.55,4.88) rectangle (2.75,5.05);
\node[anchor=west] at (2.82,4.97) {retained expected loss};
\end{tikzpicture}
\caption{Illustrative governance-deviation cost \(J_i^{\mathrm{gov}}(g;C_i)\) evaluated over discrete governance tiers. Each stacked bar decomposes total expected cost into the governance-sensitive premium, governance-control cost, and retained expected loss. In this example, the premium and retained loss decrease with stronger governance, while governance cost increases; the induced minimum occurs at \(g_i^\star=g^{(3)}\), satisfying the incentive-compatibility condition.}
\label{fig:governance_cost_objective}
\end{figure}

\subsection{Agentic-AI Insurance Contract Optimization}

The insurer's contract-design problem can now be stated as a mathematical program over the premium, layer-specific deductibles and limits, the aggregate AI limit, the binary coverage-incidence matrix, the indemnity-share matrix, and the governance obligations.

\begin{problem}[Agentic-AI insurance contract design]
For a fixed insured \(i\) with state \(s_i\) and baseline cost \(J_i^0\), choose \(C_i=(T_i,\{D_i^r,L_i^r\}_{r\in\mathcal R},A_i,\Gamma_i,\Lambda_i,\Psi_i)\) to solve
\begin{subequations}
\label{eq:contract_program}
\begin{align}
\max_{C_i}
\quad
&
\Pi_i(C_i)
=
T_i
-
\sum_{e\in\mathcal E}
q_i^eY_i^e(C_i)
-
\varrho_i(s_i,C_i)
\label{eq:insurer_problem}
\\
\text{s.t.}\quad
&
T_i
+
K_i(g_i)
+
\sum_{e\in\mathcal E}
q_i^e
\left(
x_i^e-Y_i^e(C_i)
\right)
\le
J_i^0
\label{eq:participation_constraint}
\\
&
T_i
-
\sum_{e\in\mathcal E}
q_i^eY_i^e(C_i)
\ge
\varrho_i(s_i,C_i)
\label{eq:insurer_rationality}
\\
&
g_i
\succeq
g_{\min}(\Psi_i)
\label{eq:governance_requirement}
\\
&
g_i^\star
\in
\arg\min_{\tilde g\in\mathcal G}
J_i^{\mathrm{gov}}(\tilde g;C_i)
\label{eq:ic_constraint}
\\
&
0\le \lambda_i^{e,r}\le \gamma_i^{e,r},
\qquad
\sum_{r\in\mathcal R}\lambda_i^{e,r}\le 1,
\qquad
\forall e\in\mathcal E,\;r\in\mathcal R
\label{eq:allocation_bounds}
\\
&
\gamma_i^{e,r}\in\{0,1\},
\qquad
\forall e\in\mathcal E,\;r\in\mathcal R
\label{eq:gamma_binary}
\\
&
0
\le
D_i^r
\le
L_i^r
\le
A_i,
\qquad
\forall r\in\mathcal R.
\label{eq:coverage_bounds}
\end{align}
\end{subequations}
\end{problem}

\noindent The objective \eqref{eq:insurer_problem} maximizes risk-adjusted underwriting profit. Constraint \eqref{eq:participation_constraint} is the insured's individual-rationality condition. Constraint \eqref{eq:insurer_rationality} enforces risk-loaded pricing. Constraint \eqref{eq:governance_requirement} links the policy wording \(\Psi_i\) to a minimum operational control standard, while \eqref{eq:ic_constraint} requires that maintaining the desired governance tier be optimal for the insured. Constraints \eqref{eq:allocation_bounds}--\eqref{eq:gamma_binary} formalize \(\Gamma_i\) as a binary coverage-incidence matrix and \(\Lambda_i\) as the continuous payment-allocation matrix. Constraint \eqref{eq:coverage_bounds} imposes economically meaningful layer deductibles, limits, and aggregate exposure. The formulation therefore links agent autonomy, operational authority, permissions, governance, and dependency concentration directly to pricing, coverage allocation, incentive compatibility, and claims payment.

\subsection{Insurability Region, Feasibility Monotonicity, and Governance Certification}
\label{subsec:insurability_region}

Problem~\eqref{eq:contract_program} optimizes over all contract terms. In practice, an insurer often fixes the coverage architecture---deductibles, limits, aggregate, incidence, allocation, and required governance tier---and asks only whether a premium exists that both parties accept. The case study of Section~\ref{sec:case_study} follows exactly this fixed-terms logic. We formalize it here and show that it induces a well-structured insurability region in the underwriting state space, with the monotone feasibility and governance-certification properties invoked in the conclusion.

Fix the non-premium terms \(C_i^{-T}=(\{D_i^r,L_i^r\}_{r\in\mathcal R},A_i,\Gamma_i,\Lambda_i,\Psi_i)\) together with the uninsured baseline \(J_i^0\) and the governance cost \(K_i(g_i)\). The insurer-profitability constraint~\eqref{eq:insurer_rationality} and the participation constraint~\eqref{eq:participation_constraint} bound the premium from below and above,
\begin{equation}
T_i^{\min}(s_i)=\sum_{e\in\mathcal E}q_i^eY_i^e(C_i)+\varrho_i(s_i,C_i),
\qquad
T_i^{\max}(s_i)=J_i^0-K_i(g_i)-\sum_{e\in\mathcal E}q_i^e\bigl(x_i^e-Y_i^e(C_i)\bigr).
\label{eq:premium_bounds}
\end{equation}
A contract is marketable with surplus buffer \(s_0\ge0\) if \(T_i^{\min}(s_i)\le T_i^{\max}(s_i)-s_0\), where \(s_0\) guarantees the insured a minimum saving over the uninsured baseline, as imposed in the case study. Subtracting the two bounds, the indemnity terms cancel and the marketable surplus reduces to the closed form
\begin{equation}
\Delta_i(s_i):=T_i^{\max}(s_i)-T_i^{\min}(s_i)=J_i^0-K_i(g_i)-\sum_{e\in\mathcal E}q_i^ex_i^e-\varrho_i(s_i,C_i).
\label{eq:marketable_surplus}
\end{equation}
The identity~\eqref{eq:marketable_surplus} is itself informative: the indemnity schedule \(Y_i^e\) determines \emph{where} an acceptable premium lies but not \emph{whether} the interval is nonempty. Feasibility depends only on the uninsured baseline, the governance cost, the expected gross loss, and the risk loading.

\begin{definition}[Insurability region]
\label{def:insurability_region}
For fixed non-premium terms and buffer \(s_0\ge0\), the insurability region is
\[
\mathcal S^{\mathrm{ins}}(s_0)
=\bigl\{s\in\mathcal S:\Delta(s)\ge s_0\bigr\}
=\Bigl\{s\in\mathcal S:\; K(g)+\textstyle\sum_{e\in\mathcal E}q^e(s)\,x^e(s)+\varrho(s,C)\le J^0-s_0\Bigr\}.
\]
A state \(s\) is insurable at these terms if and only if \(s\in\mathcal S^{\mathrm{ins}}(s_0)\), in which case every premium \(T\in[\,T^{\min}(s),\,T^{\max}(s)-s_0\,]\) is mutually acceptable.
\end{definition}

\noindent Insurance is thus available precisely when the expected all-in risk cost of the deployment---governance cost plus expected gross loss plus risk loading---does not exceed the value the insured can save, \(J_i^0-s_0\).

To describe how the region varies with the risk state, we order states by exposure. Write \(s\preceq_{\mathrm E}s'\) when \(s\) and \(s'\) share the same governance tier and \(s'\) is componentwise weakly more exposed: \(\alpha\preceq\alpha'\), \(\beta\le\beta'\), \(\psi_\eta(\eta)\le\psi_\eta(\eta')\), and \(R(v)\le R(v')\). The pricing primitives are \emph{exposure-monotone} if, for every event \(e\), the maps \(q^e\) and \(x^e\) are nondecreasing under \(\preceq_{\mathrm E}\) and the loading \(\varrho(\cdot,C)\) is nondecreasing under \(\preceq_{\mathrm E}\). The benchmark specifications~\eqref{eq:example_event_probability}--\eqref{eq:example_severity} satisfy this whenever their exposure coefficients are nonnegative.

\begin{proposition}[Monotone deterioration of fixed-terms insurability]
\label{prop:monotone_feasibility}
Fix the non-premium terms, \(J^0\), and \(K(g)\), and assume the pricing primitives are exposure-monotone. Then the marketable surplus \(\Delta(s)\) in~\eqref{eq:marketable_surplus} is nonincreasing under \(\preceq_{\mathrm E}\), and the lower premium bound \(T^{\min}(s)\) is nondecreasing under \(\preceq_{\mathrm E}\). Consequently the insurability region is downward closed in exposure: if \(s'\) is insurable and \(s\preceq_{\mathrm E}s'\), then \(s\) is insurable. Equivalently, along any exposure-increasing path the surplus crosses zero at most once, from feasible to infeasible, and never returns.
\end{proposition}

\begin{proof}
By exposure-monotonicity each product \(q^e(s)\,x^e(s)\) is nondecreasing under \(\preceq_{\mathrm E}\), being a product of nonnegative nondecreasing maps, and \(\varrho(s,C)\) is nondecreasing. Since \(J^0\) and \(K(g)\) are held fixed, \eqref{eq:marketable_surplus} expresses \(\Delta\) as a constant minus a sum of nondecreasing terms, so \(\Delta\) is nonincreasing. For the lower bound, \(Y_i^e(C_i)=\lambda_i^{e,r}\min\{(x_i^e-D_i^r)^+,L_i^r\}\) is nondecreasing in \(x_i^e\), hence \(\sum_e q^eY^e\) is nondecreasing and, adding the nondecreasing loading, so is \(T^{\min}\). Downward closure follows because \(s\preceq_{\mathrm E}s'\) gives \(\Delta(s)\ge\Delta(s')\ge s_0\). The single-crossing property is immediate from monotonicity of \(\Delta\) along an exposure-increasing chain.
\end{proof}

Proposition~\ref{prop:monotone_feasibility} explains why raising delegated authority, permission exposure, or dependency concentration can only erode fixed-terms feasibility, never restore it---the behavior observed numerically in the sensitivity analysis of Section~\ref{sec:case_study}. Governance acts in the opposite direction, which yields a certification threshold.

\begin{proposition}[Governance certification threshold]
\label{prop:governance_certification}
Fix the exposure coordinates \(s^{-g}=(\alpha,\beta,\eta,v)\) and the non-premium terms other than the required tier. Suppose stronger governance is net risk-reducing, i.e., the map
\[
\Phi(g):=K(g)+\sum_{e\in\mathcal E}q^e(s^{-g},g)\,x^e(s^{-g},g)+\varrho(s^{-g},g)
\]
is nonincreasing in the tier order over the relevant range. Then \(\Delta\) is nondecreasing in \(g\), and there is a minimal certifiable tier
\[
g^{\mathrm{cert}}(s^{-g})=\min\bigl\{g\in\mathcal G:\Phi(g)\le J^0-s_0\bigr\}
\]
(whenever this set is nonempty) such that the deployment is insurable at the fixed terms if and only if \(g\succeq g^{\mathrm{cert}}(s^{-g})\).
\end{proposition}

\begin{proof}
Since \(\Delta(s^{-g},g)=J^0-\Phi(g)\) by~\eqref{eq:marketable_surplus} and \(\Phi\) is nonincreasing in \(g\), \(\Delta\) is nondecreasing in \(g\). The set \(\{g:\Phi(g)\le J^0-s_0\}=\{g:\Delta\ge s_0\}\) is therefore an up-set in the tier order, so it has a least element \(g^{\mathrm{cert}}\) when nonempty, and \(g\succeq g^{\mathrm{cert}}\) is equivalent to insurability.
\end{proof}

\begin{definition}[AI-insurability certificate]
\label{def:insurability_certificate}
A deployment with exposure \(s^{-g}\) is \emph{insurability-certifiable} at the fixed coverage terms if \(g^{\mathrm{cert}}(s^{-g})\) exists and is attainable by the insured, in the sense that some tier \(g\succeq g^{\mathrm{cert}}(s^{-g})\) can be evidenced through the control scores \(\mathbf z_i\) and assigned by \(g_i=h(\mathbf z_i)\). The certificate attests that governance at tier \(g^{\mathrm{cert}}\) or higher renders the deployment insurable at the stated terms.
\end{definition}

Proposition~\ref{prop:governance_certification} converts governance from a qualitative virtue into a contractible admission condition: below \(g^{\mathrm{cert}}\) no premium clears the market, whereas at or above it a mutually acceptable contract exists. Together, Definition~\ref{def:insurability_region} and Propositions~\ref{prop:monotone_feasibility}--\ref{prop:governance_certification} show that insurability is a structured region of the risk-state space---shrinking monotonically in exposure and expanding with verifiable governance---rather than a case-by-case judgment.

\section{Insurance as AI Operating Cost and Regulatory Control}
\label{sec:insurance_regulatory_control}

The preceding framework treats insurance as a private contract between an insurer and an insured. It also carries a public-policy interpretation. When an AI deployment can inflict losses on patients, customers, employees, counterparties, infrastructure users, or the public, insurance becomes more than risk transfer: it can act as a priced condition for operating an AI system, a quasi-tax on risk creation, and a regulatory control that decides which deployments must demonstrate financial responsibility before going into use.

The interpretation is familiar in law and economics. Liability rules, safety regulation, and mandatory insurance are alternative instruments for controlling external harms when private actors do not fully internalize the social costs of their activity \cite{shavell1984liability,calabresi1970costs}. In Pigouvian terms, a charge attached to risky activity curbs excess by making decision makers face more of the expected social cost they impose on others \cite{pigou1920welfare}. Agentic AI is a natural setting for this logic, because the same technical capability can be socially valuable in one context and socially costly in another, depending on permissions, delegated authority, governance, and dependency concentration.

\subsection{Insurance as a Bundled Cost of AI Deployment}

Consider an organization deciding whether to deploy an AI agent with risk state \(s_i\). Let \(B_i(s_i)\) denote the organization's private operating benefit from the AI deployment, such as reduced labor cost, faster service, improved triage, increased throughput, or new revenue. Let \(C_i^{\mathrm{op}}(s_i)\) denote ordinary operating cost, excluding insurance and governance. If insurance is optional, the organization may compare the AI benefit with technical operating cost and expected uninsured loss. If insurance is bundled into the legal or commercial cost of using AI, the adoption calculation changes.

For a deployment insured under contract \(C_i\), the private annual cost of using AI can be written as
\begin{equation}
\mathcal C_i^{\mathrm{AI}}(s_i,C_i)
=
C_i^{\mathrm{op}}(s_i)
+K_i(g_i)
+T_i
+\sum_{e\in\mathcal E}q_i^e\bigl(x_i^e-Y_i^e(C_i)\bigr).
\label{eq:ai_bundled_cost}
\end{equation}
The term \(T_i\) is then not merely an insurance premium from an accounting perspective. It becomes part of the all-in cost of deploying agentic AI. If a vendor bundles insurance into an AI product, this cost may appear as a licensing surcharge, a per-agent fee, a per-action fee, or a sector-specific compliance charge. For example, if \(N_i^{\mathrm{act}}\) denotes the expected annual number of covered agent actions, the premium can be converted into an action-level insurance cost \(c_i^{\mathrm{act}}=T_i/N_i^{\mathrm{act}}\). A buyer then experiences insurance as part of the marginal cost of operating the AI workflow.

This cost has an incentive effect. The organization deploys the agent only if the net value \(B_i(s_i)-\mathcal C_i^{\mathrm{AI}}(s_i,C_i)\) exceeds the non-AI alternative, so higher premiums, governance expenditures, and retained losses can discourage deployment. That is not necessarily a defect: when an AI system creates high expected external harm and low private value, a high insurance cost screens out undesirable deployment. But if premiums are poorly calibrated, unavailable, or inflated by legal uncertainty and sparse claims data, the same mechanism can deter socially valuable uses. Insurance can therefore discipline risky AI adoption, yet also become a barrier to innovation whenever the price of coverage substantially exceeds the risk the deployment creates.

\subsection{Premiums as Quasi-Pigouvian Risk Prices}

The premium can be interpreted as a private-market analogue of a Pigouvian tax when it increases with the expected losses and external risks generated by the deployment. In the model above, the risk-loaded premium condition requires \(T_i\ge \sum_{e\in\mathcal E}q_i^eY_i^e(C_i)+\varrho_i(s_i,C_i)\). When the indemnity \(Y_i^e(C_i)\) approximates harm that would otherwise be borne by third parties, customers, or the public, the premium forces the AI operator to internalize a priced portion of that harm.

This interpretation is clearest when the premium is monotone in the exposure components of the risk state. If higher operational authority \(\beta_i\), broader weighted permissions \(\psi_\eta(\eta_i)\), or greater dependency concentration \(R(v_i)\) increases event probabilities, severities, or risk loadings, then the premium increases with risk-creating activity. Conversely, stronger governance \(g_i\) can reduce the charge by lowering probabilities, severities, deductibles, or risk loadings. The premium schedule therefore acts as a market-based control signal:
\[
T_i
=
\tau(s_i,C_i)
\quad
\text{with}\quad
\tau \text{ increasing in exposure and decreasing in verified governance.}
\]
Unlike a public tax, the premium is paid to an insurer rather than the government, and it finances risk transfer, claims adjustment, capital cost, and monitoring. Nevertheless, from the AI user's perspective, it has a tax-like effect because it raises the private cost of risky AI deployment and can reduce demand for high-risk configurations.

\begin{proposition}[Adoption effect of bundled insurance cost]
\label{prop:adoption_effect}
Fix the non-AI outside option and the AI deployment benefit \(B_i(s_i)\). If insurance is mandatory or commercially bundled into the deployment cost, then any increase in \(T_i\), \(K_i(g_i)\), or expected residual loss weakly decreases the set of risk states for which deploying the AI system is privately profitable.
\end{proposition}

\begin{proof}
The deployment is privately profitable only when \(B_i(s_i)-\mathcal C_i^{\mathrm{AI}}(s_i,C_i)\) exceeds the outside option. In \eqref{eq:ai_bundled_cost}, the cost \(\mathcal C_i^{\mathrm{AI}}\) is increasing in \(T_i\), \(K_i(g_i)\), and expected residual loss. Increasing any of these terms weakly lowers net private value and therefore can only shrink, not expand, the set of states satisfying the adoption inequality.
\end{proof}

The proposition formalizes the basic adoption tradeoff. Insurance can be a useful social instrument because it discourages deployments whose benefits do not justify their risk. But it also creates a policy-design problem: the mandate should be targeted enough that the cost falls primarily on deployments capable of material third-party harm rather than on low-risk assistive uses.

\subsection{Mandatory Insurance and Financial Responsibility}

A regulator can use the risk-state representation to specify which AI deployments must be insured. Let \(d_i\) denote nontechnical context, such as sector, population affected, criticality, transaction value, and whether the deployment affects health, employment, credit, education, public safety, financial transfers, or physical systems. A mandate can be represented by an indicator
\[
M(s_i,d_i)
=
\begin{cases}
1, & \text{if deployment \(i\) must carry agentic-AI insurance},\\
0, & \text{otherwise}.
\end{cases}
\]
The mandate may be triggered by autonomy category, authority, permission exposure, dependency concentration, sector, or combinations of these factors. For example, a regulator may require insurance when \(\alpha_i\succeq\alpha^{(2)}\), \(\beta_i\ge\bar\beta\), \(\psi_\eta(\eta_i)\ge\bar\psi\), or the deployment operates in a high-impact sector. A stricter rule may require insurance for any cyber-physical agent or any AI workflow authorized to initiate financial transactions, modify critical records, or make operational decisions without human approval.

\begin{definition}[Mandated insurability requirement]
\label{def:mandated_insurability}
A deployment \(i\) satisfies a mandated insurability requirement if \(M(s_i,d_i)=0\), or if \(M(s_i,d_i)=1\) and there exists an admissible contract \(C_i\in\mathcal C_i\) such that the required premium is paid, the aggregate limit satisfies \(A_i\ge A_{\min}(s_i,d_i)\), the governance tier satisfies \(g_i\succeq g_{\min}(\Psi_i)\), and the coverage incidence matrix \(\Gamma_i\) includes all event-layer pairs required by law or regulation.
\end{definition}

This definition separates two questions that are often conflated. The first question is whether an organization can obtain insurance in the market. The second is whether the organization is legally permitted to operate the AI system without insurance. A mandate makes insurability a precondition for deployment. If no admissible contract exists because the risk is too high, the coverage is unavailable, or the required governance tier is not met, then the deployment fails the financial-responsibility requirement even if the operator would otherwise prefer to use the AI system.

\subsection{Insurance as a Control Mechanism}

Mandatory insurance can therefore operate as a regulatory control. It does not directly prescribe every technical design choice. Instead, it requires an AI operator to satisfy underwriting, governance, coverage, and monitoring conditions before deployment. The control is implemented through several linked mechanisms.

First, the mandate defines the boundary of covered AI activity. The permission profile \(\eta_i\), operational authority \(\beta_i\), and deployment context \(d_i\) determine whether insurance is required. Second, the insurer evaluates the state \(s_i\) and prices the policy. Third, the policy covenant \(\Psi_i\) requires ongoing governance controls, notice obligations, and evidence retention. Fourth, failure to maintain the underwritten state can trigger repricing, higher deductibles, loss of premium credits, suspension of coverage for newly enabled permissions, nonrenewal, or regulatory noncompliance.

This gives insurance a dual role. Ex ante, it screens deployments by making high-risk AI more expensive, or uninsurable until governance improves. Ex post, it builds a claims and monitoring infrastructure that records losses, identifies failure modes, and generates data for future pricing and regulation. Over time, the feedback loop sharpens both actuarial calibration and public oversight.

The policy challenge is calibration. Too broad a mandate becomes a general AI tax that deters low-risk productivity tools and small organizations; too narrow a mandate lets high-risk deployments operate without adequate financial responsibility. A principled mandate should therefore be risk-based, attaching its strongest requirements to systems with high autonomy, high delegated authority, broad external-state permissions, high-impact sectors, cyber-physical consequences, or systemic dependency concentration. In this sense agentic-AI insurance is not merely a private financial product; it is a candidate governance layer for aligning AI deployment incentives with social risk.

\section{Case Study: Agentic-AI Insurance for Clinical Care Coordination}
\label{sec:case_study}

To illustrate the framework in practice, we consider a healthcare system that deploys an autonomous clinical care-coordination agent to support physicians, nurses, and administrative staff. The agent reviews patient messages, summarizes electronic health records, prioritizes incoming requests, schedules appointments, drafts patient communications, and escalates potentially urgent cases for clinical review. It cannot prescribe medication or independently discharge patients, but it does produce persistent external-state changes through scheduling systems, communication channels, and workflow-management tools. The deployment therefore falls within the autonomous digital-agent category and carries insurable risks that extend well beyond traditional cybersecurity concerns.

The hospital's deployment is represented by the state vector \(s_i=(\alpha_i,\beta_i,\eta_i,g_i,v_i)\), where \(\alpha_i=\alpha^{(2)}\) denotes an autonomous digital agent, \(\beta_i=0.45\) means that roughly 45 percent of weighted authorized actions may execute without human approval, and \(\eta_i=(1,1,0,0,0)\) indicates that the system can communicate with patients and schedule appointments but cannot prescribe medication, modify billing systems, or control physical devices. Under the illustrative permission weights in Table~\ref{tab:permission_weights}, this profile has weighted permission exposure \(\psi_\eta(\eta_i)=3\), reflecting email and scheduling rights without higher-risk billing, record-modification, or device-control permissions. The hospital maintains extensive governance controls including clinician review, audit logging, prompt-injection filtering, incident monitoring, and rollback capabilities, so the underwriting rule assigns it to governance tier \(g_i=g^{(3)}\). Finally, the deployment relies on a single cloud-hosted foundation-model provider, resulting in a dependency concentration measure of \(R(v_i)=0.40\).

The reduced healthcare event set is \(\mathcal E^{\mathrm{hc}}=\{e_H,e_P,e_F,e_D,e_O,e_C\}\), where \(e_H\) denotes hallucination, \(e_P\) prompt injection, \(e_F\) agentic fraud, \(e_D\) model drift, \(e_O\) dependency outage, and \(e_C\) cyber-physical harm. Table~\ref{tab:healthcare_losses} reports representative annual probabilities \(q_i^e\), gross losses \(x_i^e\), and the active coverage layer used in the simplified numerical calculation.

\begin{table}[ht]
\centering
\caption{Representative Agentic-AI Healthcare Loss Scenarios}
\label{tab:healthcare_losses}
\begin{tabular}{llcc}
\hline
Event \(e\) & Active layer \(r\) & \(q_i^e\) & \(x_i^e\) \\
\hline
Hallucination \(e_H\) & TechE\&O & 0.060 & \$250,000 \\
Prompt injection \(e_P\) & Cyber & 0.025 & \$400,000 \\
Agentic fraud \(e_F\) & Mixed & 0.010 & \$300,000 \\
Model drift \(e_D\) & Performance & 0.040 & \$150,000 \\
Dependency outage \(e_O\) & Mixed & 0.080 & \$100,000 \\
Cyber-physical harm \(e_C\) & GL & 0.002 & \$2,000,000 \\
\hline
\end{tabular}
\end{table}

The expected annual loss before insurance is

\begin{equation}
\begin{aligned}
\mathbb E[X_i]
&=
\sum_{e\in\mathcal E^{\mathrm{hc}}}
q_i^e x_i^e
\\
&=
0.060(250,000)
+0.025(400,000)
+0.010(300,000)
\\
&\quad
+0.040(150,000)
+0.080(100,000)
+0.002(2,000,000)
\\
&=
46,000.
\end{aligned}
\label{eq:hospital_expected_loss}
\end{equation}

\noindent The insurer seeks to design \(C_i=(T_i,\{D_i^r,L_i^r\}_{r\in\mathcal R},A_i,\Gamma_i,\Lambda_i,\Psi_i)\). The objective is to maximize risk-adjusted underwriting profit while ensuring that the hospital voluntarily participates, the insurer earns at least the required risk loading, and the hospital maintains the required governance controls. The resulting optimization problem is

\begin{subequations}
\begin{align}
\max_{C_i}
\quad
&
\Pi_i(C_i)
=
T_i
-
\sum_{e\in\mathcal E^{\mathrm{hc}}}
q_i^eY_i^e(C_i)
-
\varrho_i(s_i,C_i)
\label{eq:healthcare_opt_obj}
\\
\text{s.t.}\quad
&
T_i
+
K_i(g_i)
+
\sum_{e\in\mathcal E^{\mathrm{hc}}}
q_i^e
(x_i^e-Y_i^e(C_i))
\le
J_i^0
\label{eq:healthcare_opt_part}
\\
&
T_i
-
\sum_{e\in\mathcal E^{\mathrm{hc}}}
q_i^eY_i^e(C_i)
\ge
\varrho_i(s_i,C_i)
\label{eq:healthcare_opt_profit}
\\
&
g_i
\succeq
g_{\min}(\Psi_i)
\label{eq:healthcare_opt_gov}
\\
&
0\le \lambda_i^{e,r}\le \gamma_i^{e,r},
\qquad
\sum_{r\in\mathcal R}\lambda_i^{e,r}\le 1,
\qquad
\forall e\in\mathcal E^{\mathrm{hc}},\;r\in\mathcal R
\label{eq:healthcare_opt_alloc}
\\
&
\gamma_i^{e,r}\in\{0,1\},
\qquad
\forall e\in\mathcal E^{\mathrm{hc}},\;r\in\mathcal R
\label{eq:healthcare_opt_gamma}
\\
&
0
\le
D_i^r
\le
L_i^r
\le
A_i,
\qquad
\forall r\in\mathcal R.
\label{eq:healthcare_opt_bounds}
\end{align}
\end{subequations}

\noindent Constraint \eqref{eq:healthcare_opt_part} ensures that the hospital prefers purchasing insurance to remaining uninsured. Constraint \eqref{eq:healthcare_opt_profit} imposes the insurer's risk-loaded profitability requirement, while \eqref{eq:healthcare_opt_gov} enforces the governance obligations embedded in the policy. Constraints \eqref{eq:healthcare_opt_alloc}--\eqref{eq:healthcare_opt_gamma} require \(\Gamma_i\) to be a binary coverage-incidence matrix and \(\Lambda_i\) to allocate payment only to active coverage layers.

To illustrate how the optimization problem is solved, suppose that the governance-obligation package \(\Psi_i\) is fixed by regulatory and organizational requirements. In this case study, \(\Psi_i\) is a policy-covenant package requiring \(g_{\min}(\Psi_i)=g^{(3)}\), clinician approval for high-acuity patient communications, least-privilege access to scheduling and messaging tools, retention of complete and tamper-evident action logs for twenty-four months, monthly reporting of autonomous-action rates, quarterly prompt-injection testing, rollback capability, dependency-continuity procedures for the foundation-model provider, insurer notice within seventy-two hours after unauthorized tool execution or suspected patient-data disclosure, an incident report within fourteen days, and a remediation plan within thirty days. The computational implementation in \texttt{case\_study\_optimization.py} solves a finite-menu version of \eqref{eq:healthcare_opt_obj}--\eqref{eq:healthcare_opt_bounds}. It searches two governance choices, \(g^{(3)}\) and \(g^{(4)}\); five common deductible choices from \$0 to \$100,000 in \$25,000 increments; four common layer limits from \$250,000 to \$1,000,000; and three aggregate limits from \$1,000,000 to \$1,500,000. To avoid nonmarketable thin coverage, the menu imposes an expected indemnity ratio of at least 75 percent, a cyber-physical event payment of at least \$400,000, and an insured surplus of at least \$5,000 relative to the uninsured baseline. The risk-loading rule used in the computation is \(\varrho_i=9000+0.05\,\mathbb E[Y_i(C_i)]+0.002(A_i-1,000,000)^++2500-c(g_i^\star)\), where \(c(g^{(4)})=1500\) and \(c(g^{(3)})=0\). This loading captures fixed underwriting expense, claims volatility, dependency concentration, aggregate-limit capital cost, and a governance credit for the stronger tier. For the simplified numerical calculation, the active incidence entries of \(\Gamma_i\) are exactly the event-layer pairs listed in Table~\ref{tab:healthcare_losses}. For each listed pair \((e,r)\), set \(\gamma_i^{e,r}=1\) and \(\lambda_i^{e,r}=1\); for all other layers \(r'\ne r\), set \(\gamma_i^{e,r'}=\lambda_i^{e,r'}=0\). Thus \(\Gamma_i\) is a binary matrix with one active entry in each event row, while \(\Lambda_i\) assigns the full payable loss to that active layer. Mixed-cause claims can be represented by activating multiple entries in the same row and choosing corresponding shares in \(\Lambda_i\).

For any candidate contract, first compute the event payment before the annual aggregate:

\begin{equation}
\widehat Y_i^{e,r}(C_i)
=
\lambda_i^{e,r}
\min\{(x_i^e-D_i^r)^+,L_i^r\},
\qquad
\widehat Y_i^e(C_i)=\sum_{r\in\mathcal R}\widehat Y_i^{e,r}(C_i).
\label{eq:case_indemnity}
\end{equation}

\noindent where \((x)^+=\max\{x,0\}\). Let \(\widehat Y_i(C_i)=\sum_{e\in\mathcal E^{\mathrm{hc}}}\widehat Y_i^e(C_i)\). The annual aggregate is then applied by \(Y_i^e(C_i)=\widehat Y_i^e(C_i)\min\{1,A_i/\widehat Y_i(C_i)\}\) when \(\widehat Y_i(C_i)>0\), and \(Y_i^e(C_i)=0\) otherwise. Because the illustrative calculation has one active layer for each event and common active-layer terms \(D_i^r=\$25,000\) and \(L_i^r=\$500,000\), the six pre-aggregate event payments are \$225,000, \$375,000, \$275,000, \$125,000, \$75,000, and \$500,000. Their sum is \$1,575,000. For the selected annual aggregate \(A_i=\$1,500,000\), all event payments are therefore multiplied by \(1,500,000/1,575,000\), yielding the indemnity values reported in Table~\ref{tab:healthcare_indemnities}.

\begin{table}[ht]
\centering
\caption{Indemnity Payments Under the Candidate Contract}
\label{tab:healthcare_indemnities}
\begin{tabular}{lccc}
\hline
Event & Active layer & Gross loss & \(Y_i^e(C_i)\) \\
\hline
Hallucination \(e_H\) & TechE\&O & \$250,000 & \$214,286 \\
Prompt injection \(e_P\) & Cyber & \$400,000 & \$357,143 \\
Agentic fraud \(e_F\) & Mixed & \$300,000 & \$261,905 \\
Model drift \(e_D\) & Performance & \$150,000 & \$119,048 \\
Dependency outage \(e_O\) & Mixed & \$100,000 & \$71,429 \\
Cyber-physical harm \(e_C\) & GL & \$2,000,000 & \$476,190 \\
\hline
\end{tabular}
\end{table}

For example, the hallucination loss of \$250,000 has active TechE\&O incidence \(\gamma_i^{e_H,\mathrm{TechEO}}=1\), so it produces a pre-aggregate payment of \$225,000 and an aggregate-adjusted payment of \(Y_i^{e_H}(C_i)=\$214,286\). The catastrophic cyber-physical loss of \$2,000,000 has active GL incidence \(\gamma_i^{e_C,\mathrm{GL}}=1\); it is capped by the layer limit at \$500,000 before the annual aggregate and becomes \(Y_i^{e_C}(C_i)=\$476,190\) after aggregate allocation.

Substituting all indemnity values into the insurer's expected payout calculation gives

\begin{align}
\mathbb E[Y_i(C_i)]
&=
\sum_{e\in\mathcal E^{\mathrm{hc}}}
q_i^eY_i^e(C_i)
\nonumber\\
&=
0.060(214,286)
+0.025(357,143)
+0.010(261,905)
\nonumber\\
&\quad
+0.040(119,048)
+0.080(71,429)
+0.002(476,190)
\nonumber\\
&=
\$35,833.
\label{eq:expected_indemnity}
\end{align}

\begin{table}[ht]
\centering
\caption{Computational Solution of the Healthcare Contract-Design Problem}
\label{tab:case_study_computational_solution}
\begin{tabular}{ll}
\hline
Quantity & Computational solution \\
\hline
Governance tier & \(g^{(3)}\) \\
Layer deductible \(D_i^r\) & \$25,000 \\
Layer limit \(L_i^r\) & \$500,000 \\
Aggregate AI limit \(A_i\) & \$1,500,000 \\
Expected gross loss \(\mathbb E[X_i]\) & \$46,000 \\
Expected indemnity \(\mathbb E[Y_i(C_i)]\) & \$35,833 \\
Expected coverage ratio & 77.9\% \\
Expected residual loss & \$10,167 \\
Risk loading \(\varrho_i(s_i,C_i)\) & \$14,292 \\
Feasible premium interval & \$50,125--\$54,833 \\
Selected premium \(T_i^\star\) & \$54,833 \\
Risk-adjusted underwriting profit & \$4,708 \\
Hospital expected annual cost & \$95,000 \\
Hospital savings versus uninsured baseline & \$5,000 \\
\hline
\end{tabular}
\end{table}
 
\begin{figure}[H]
\centering
\includegraphics[width=\textwidth]{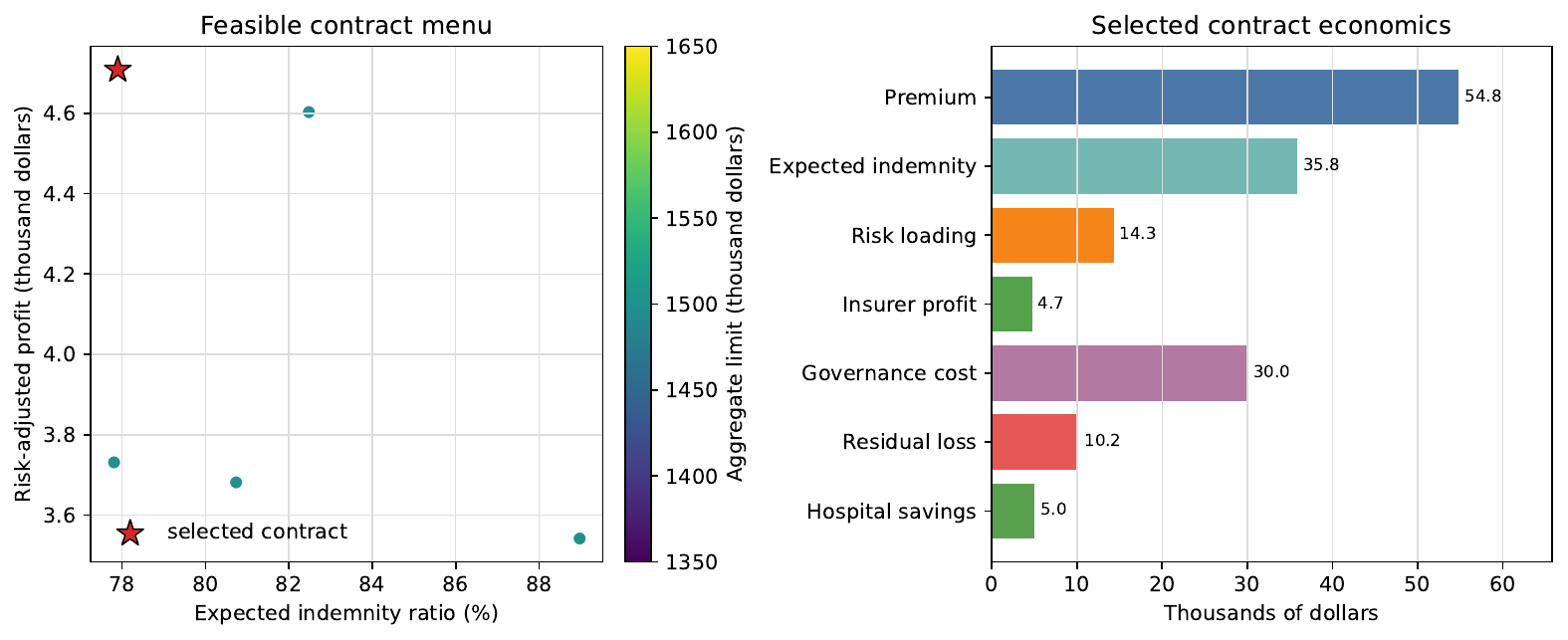}
\caption{Computational solution of the healthcare case study. The left panel plots all feasible contracts in the finite menu by expected indemnity ratio and risk-adjusted underwriting profit, with color indicating the aggregate limit. The right panel decomposes the selected contract economics into premium, expected indemnity, risk loading, insurer profit, governance cost, residual loss, and hospital savings.}
\label{fig:case_study_optimization_results}
\end{figure}

\noindent Equation \eqref{eq:expected_indemnity} determines the actuarially fair component of the premium for the selected contract. The computational solution selects \(g_i^\star=g^{(3)}\), \(D_i^r=\$25,000\), \(L_i^r=\$500,000\), and \(A_i=\$1,500,000\). Under the risk-loading rule above, \(\varrho_i(s_i,C_i)=\$14,292\), so the profitability constraint requires \(T_i\ge35,833+14,292=\$50,125\). This establishes the minimum premium required to satisfy the insurer's risk-adjusted profitability requirement for the selected coverage terms.

Next, the insurer evaluates the hospital's participation constraint. Assume that maintaining the governance controls associated with tier \(g_i=g^{(3)}\) requires annual expenditures of \(K_i(g_i)=\$30,000\).

The expected residual loss retained by the hospital after insurance is \(\mathbb E[X_i-Y_i(C_i)]=46,000-35,833=\$10,167\). Substituting these quantities into the hospital's participation-cost function yields \(J_i^{\mathrm{part}}(C_i)=T_i+30,000+10,167\). Suppose that operating without insurance would result in weaker governance and an expected annual cost of \(J_i^0=\$100,000\). The computational case study imposes a minimum insured surplus of \$5,000, so the participation condition is \(T_i+40,167\le95,000\), which implies \(T_i\le\$54,833\). Combining the minimum and maximum premium conditions yields the feasible premium interval \(50,125\le T_i\le54,833\). The insurer-profit-maximizing premium within this marketable interval is therefore \(T_i^\star=\$54,833\). Substituting this premium into the objective function gives \(\Pi_i(C_i)=54,833-35,833-14,292=\$4,708\). The hospital's expected annual cost becomes \(J_i^{\mathrm{part}}(C_i)=54,833+30,000+10,167=\$95,000\), which is \$5,000 below the uninsured baseline. The optimized contract therefore benefits both parties: the hospital reduces its expected total cost while maintaining strong governance controls, and the insurer receives compensation for assuming agentic-AI risk.

The same code can be used to compare the impact of individual underwriting variables. Table~\ref{tab:case_study_sensitivity} and Figure~\ref{fig:case_study_sensitivity_results} recompute the finite-menu problem under one-factor changes in operational authority, weighted permission exposure, dependency concentration, and governance tier. Except for the final governance-upgrade row, the comparisons hold the governance requirement at \(g^{(3)}\) so that the effect of the changed variable is isolated. The values should be read as scenario-based underwriting calculations rather than empirical estimates.

\begin{table}[ht]
\centering
\small
\caption{Sensitivity of Optimized Contract Economics to Underwriting Variables}
\label{tab:case_study_sensitivity}
\begin{tabular}{llrrrr}
\hline
Scenario & Change & \(\mathbb E[X_i]\) & \(T_i^{\min}\) & \(T_i^{\max}\) & Result \\
\hline
Lower authority & \(\beta_i=0.20\) & \$41,690 & \$45,562 & \$54,798 & profit \$9,235 \\
Baseline & \(\beta_i=0.45,\ \psi_\eta=3,\ R=0.40\) & \$46,000 & \$50,125 & \$54,833 & profit \$4,708 \\
Higher authority & \(\beta_i=0.75\) & \$52,029 & \$55,233 & \$53,669 & gap \$1,564 \\
Broader permissions & \(\psi_\eta(\eta_i)=7\) & \$49,331 & \$51,437 & \$52,752 & profit \$1,315 \\
High dependency & \(R(v_i)=0.80\) & \$51,972 & \$56,778 & \$52,817 & gap \$3,961 \\
Stronger governance & \(g_i=g^{(4)}\) & \$33,948 & \$38,735 & \$42,467 & profit \$3,731 \\
\hline
\end{tabular}
\end{table}
 
\begin{figure}[H]
\centering
\includegraphics[width=\textwidth]{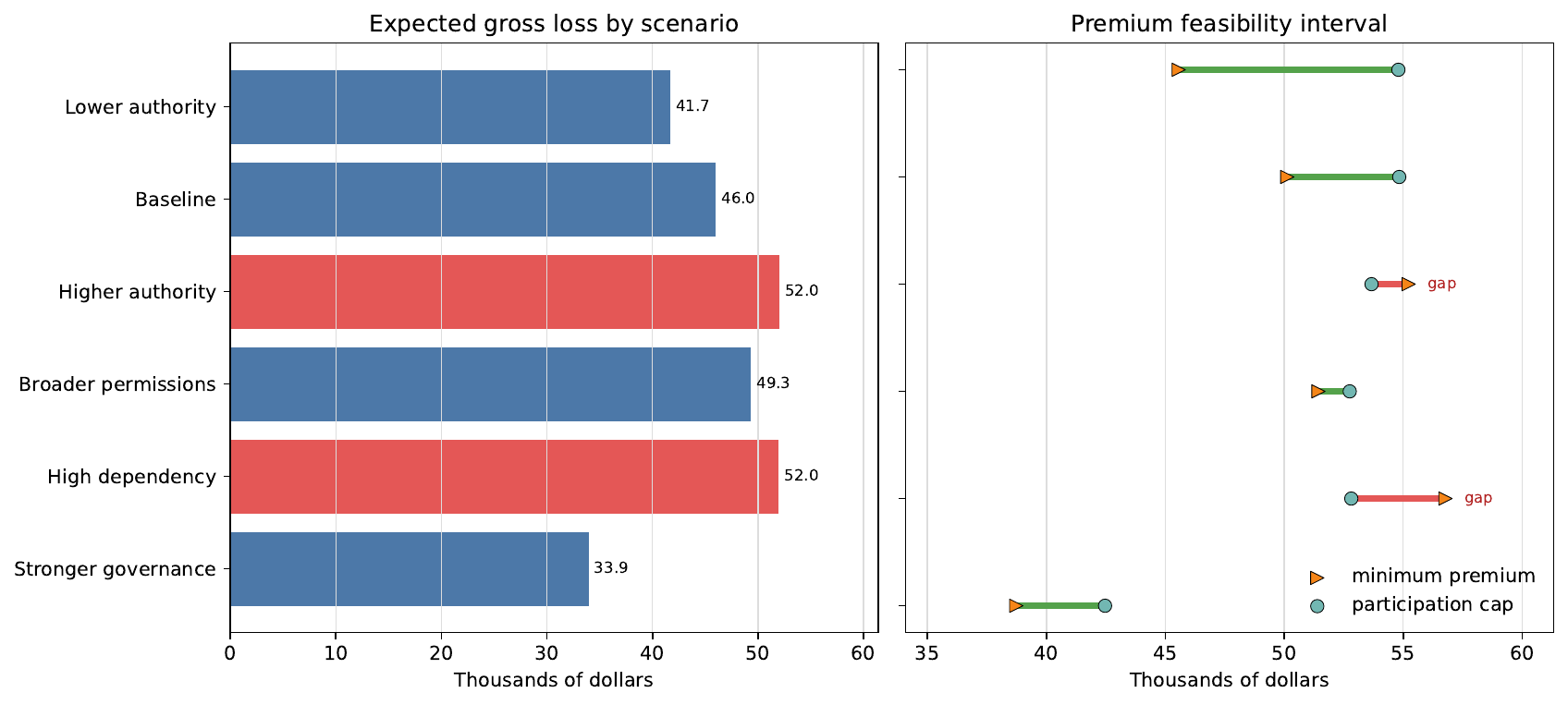}
\caption{Sensitivity of the optimized healthcare contract to selected underwriting variables. The left panel shows expected gross loss under each scenario. The right panel compares the minimum risk-loaded premium \(T_i^{\min}\) with the participation cap \(T_i^{\max}\). A scenario becomes infeasible when the minimum premium exceeds the participation cap.}
\label{fig:case_study_sensitivity_results}
\end{figure}

Reducing delegated authority from \(\beta_i=0.45\) to \(\beta_i=0.20\) lowers expected gross loss from \$46,000 to \$41,690 and raises feasible underwriting profit to \$9,235. Increasing authority to \(\beta_i=0.75\) instead raises expected gross loss to \$52,029 and opens a \$1,564 insurability gap under the fixed participation baseline. Expanding weighted permission exposure from \(\psi_\eta(\eta_i)=3\) to \(7\) remains feasible but cuts risk-adjusted profit from \$4,708 to \$1,315, and raising dependency concentration to \(R(v_i)=0.80\) opens a larger \$3,961 gap by increasing both expected loss and the dependency-related loading. Upgrading the governance requirement to \(g^{(4)}\), by contrast, reduces expected gross loss to \$33,948 and keeps the contract feasible with \$3,731 of risk-adjusted profit after the higher governance cost. These one-factor movements instantiate the structural results of Section~\ref{subsec:insurability_region}: every exposure increase shrinks the marketable surplus toward infeasibility, as in Proposition~\ref{prop:monotone_feasibility}, while the governance upgrade expands it, as in Proposition~\ref{prop:governance_certification}.

This example shows how the framework turns qualitative notions---autonomy category, operational authority, governance maturity tier, external-state permissions, and dependency concentration---into a quantitative contract-design problem. Rather than picking premiums and limits by hand, the insurer derives every contract parameter from a constrained optimization that jointly weighs expected losses, governance incentives, participation constraints, and risk-adjusted profitability. The result is a rigorous foundation for underwriting and pricing agentic-AI systems in healthcare settings.

\section{Designed Contract for the Healthcare Case Study}
\label{sec:designed_healthcare_contract}

The numerical case study in Section~\ref{sec:case_study} solves the contract-design problem; an insurance paper should also show what the resulting policy looks like in practice. This section translates the optimized healthcare example into a realistic contract schedule for the clinical care-coordination agent. The aim is not a complete legal policy form, but a demonstration of how the mathematical objects \(T_i\), \(D_i^r\), \(L_i^r\), \(A_i\), \(\Gamma_i\), \(\Lambda_i\), and \(\Psi_i\) become concrete policy terms.

The designed contract is a twelve-month agentic-AI endorsement attached to a healthcare cyber and technology E\&O insurance package. It covers the specific deployed care-coordination agent described in the case study, with risk state \(\alpha_i=\alpha^{(2)}\), \(\beta_i=0.45\), \(\psi_\eta(\eta_i)=3\), \(g_i=g^{(3)}\), and \(R(v_i)=0.40\). The policy is priced at the optimized premium \(T_i^\star=\$54{,}833\). It uses a \$25,000 per-event deductible, a \$500,000 per-event active-layer limit, and a \$1,500,000 annual AI aggregate limit. These values are the selected solution reported in Table~\ref{tab:case_study_computational_solution}.

\begin{table}[ht]
\centering
\caption{Issued Contract Schedule for the Clinical Care-Coordination Agent}
\label{tab:issued_healthcare_contract}
\small
\begin{tabularx}{\textwidth}{>{\raggedright\arraybackslash}p{0.28\textwidth}>{\raggedright\arraybackslash}X}
\hline
Contract item & Designed policy term \\
\hline
Insured AI system & Clinical care-coordination agent used for patient-message triage, appointment scheduling, patient communications, workflow routing, and escalation support. \\
Policy period & Twelve months, with renewal subject to updated risk-state review and control evidence. \\
Premium & \(T_i^\star=\$54{,}833\). This equals the participation cap in the optimized marketable interval and yields risk-adjusted underwriting profit of \$4,708. \\
Required governance tier & \(g_i^\star=g^{(3)}\), with continued eligibility conditioned on maintaining the covenant package \(\Psi_i\). \\
Deductible & \(D_i^r=\$25{,}000\) for each covered event and active coverage layer. \\
Per-event layer limit & \(L_i^r=\$500{,}000\) for each active layer shown in Table~\ref{tab:issued_healthcare_layer_schedule}. \\
Annual AI aggregate limit & \(A_i=\$1{,}500{,}000\) across all covered agentic-AI events during the policy period. \\
Expected economics & Expected gross loss is \$46,000, expected indemnity is \$35,833, expected residual loss is \$10,167, and the risk loading is \$14,292. \\
Covered deployment boundary & The agent may use patient messaging and scheduling tools. Prescribing, billing changes, EHR modification beyond approved workflow notes, procurement, and physical-device control are not covered unless endorsed after underwriting review. \\
\hline
\end{tabularx}
\end{table}

The coverage grant is event-specific. Table~\ref{tab:issued_healthcare_layer_schedule} gives the corresponding binary incidence entries \(\gamma_i^{e,r}\) and payment shares \(\lambda_i^{e,r}\). For each listed event, the active layer has \(\gamma_i^{e,r}=1\) and \(\lambda_i^{e,r}=1\); nonlisted layer-event pairs have \(\gamma_i^{e,r}=0\) and \(\lambda_i^{e,r}=0\). This is a simple one-layer schedule chosen for the case study. In a real mixed-cause claim, the same structure could allocate one event across several active layers by using several positive \(\lambda_i^{e,r}\) values whose sum is at most one.

\begin{table}[ht]
\centering
\caption{Coverage-Layer Schedule for the Issued Healthcare Contract}
\label{tab:issued_healthcare_layer_schedule}
\small
\begin{tabularx}{\textwidth}{>{\raggedright\arraybackslash}p{0.20\textwidth}>{\raggedright\arraybackslash}p{0.15\textwidth}>{\raggedright\arraybackslash}p{0.17\textwidth}>{\raggedright\arraybackslash}X}
\hline
Covered event & Active layer & Modeled payment & Practical coverage interpretation \\
\hline
Hallucination \(e_H\) & TechE\&O & \$214,286 & Erroneous AI-generated care-coordination recommendation, message, routing decision, or escalation failure relied upon in clinical operations. \\
Prompt injection \(e_P\) & Cyber & \$357,143 & Adversarial prompt or malicious input causes unauthorized disclosure, tool invocation, or security-policy bypass. \\
Agentic fraud \(e_F\) & Mixed & \$261,905 & Autonomous communication or workflow action contributes to fraudulent transaction, deceptive communication, or improper benefit routing. \\
Model drift \(e_D\) & Performance & \$119,048 & Degradation in triage, prioritization, or routing performance causes measurable operational loss or required remediation. \\
Dependency outage \(e_O\) & Mixed & \$71,429 & Foundation-model, cloud, connector, or external API outage disrupts covered care-coordination operations. \\
Cyber-physical harm \(e_C\) & GL & \$476,190 & Covered AI workflow contributes to bodily injury, delayed escalation of urgent care, or related liability; payment is capped by the layer limit and then adjusted by the annual aggregate. \\
\hline
\end{tabularx}
\end{table}

The governance covenant \(\Psi_i\) is central to the policy because the premium and coverage terms are justified only for the underwritten risk state. The issued contract requires \(g_i^\star=g^{(3)}\), clinician approval for high-acuity messages and safety-critical workflow changes, least-privilege access to messaging and scheduling systems, tamper-evident logs of prompts, outputs, tool calls, approvals, and actions for twenty-four months, monthly reporting of the autonomous-action rate used to estimate \(\beta_i\), quarterly prompt-injection and adversarial testing, and documented rollback procedures. The policy also requires dependency-continuity procedures for the foundation-model provider and any critical external connector.

The permission covenant fixes the underwritten permission boundary \(\eta_i=(1,1,0,0,0)\). The insured must notify the insurer before enabling new external-state permissions such as billing, record modification, procurement, prescribing support, or device control. A material expansion can be defined by a positive increase in \(\psi_\eta(\eta_i)\) beyond an insurer-specified threshold, or by activation of any high-risk permission class. Until the expansion is underwritten and endorsed, losses arising from the newly enabled permission are outside the designed contract boundary.

The authority covenant fixes the underwritten operating authority at \(\beta_i\le0.45\). The hospital must report the monthly fraction of weighted actions that execute without human approval. If telemetry shows a sustained increase above the underwritten threshold, the insurer may require remediation, re-underwrite the policy, remove governance credits, or increase deductibles for future events. This term is important because two agents with the same autonomy category \(\alpha^{(2)}\) can have very different risk if one only drafts recommended actions and the other executes most actions automatically.

Claims duties are also stated in operational terms. The insured must provide notice within seventy-two hours after unauthorized tool execution, suspected patient-data disclosure, dependency failure causing material workflow disruption, or any agent-related event reasonably expected to give rise to a claim. It must preserve logs and approval records, submit an incident report within fourteen days, and submit a remediation plan within thirty days. These duties are not merely administrative. They preserve the evidence needed to identify the event \(e\), confirm the active coverage layer \(r\), evaluate the incidence matrix \(\Gamma_i\), determine any payment share in \(\Lambda_i\), and compute the indemnity \(Y_i^e(C_i)\).

This contract shows how the optimization output becomes an insurable product. The insurer is compensated for expected indemnity and risk loading; the hospital obtains an expected coverage ratio of 77.9 percent and cuts its expected annual cost by \$5,000 against the uninsured baseline; and the governance covenants hold the deployed agent within the priced risk state. The policy is thus far more than a premium and a limit---it is a coupled pricing, coverage, monitoring, and governance mechanism for one specific agentic-AI deployment.

\section{Automated Agentic-AI Insurance Workflow}
\label{sec:automated_workflow}

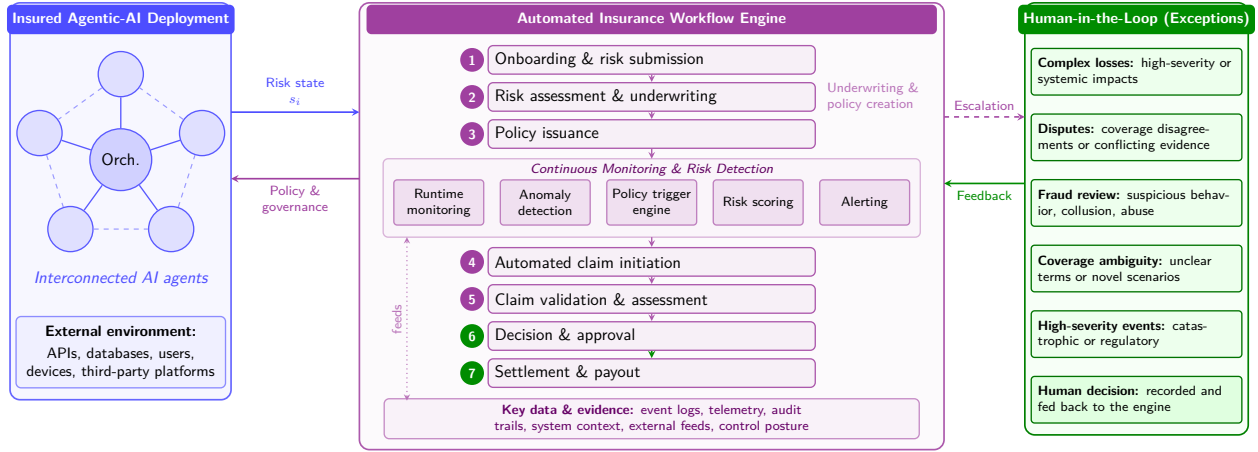
\begin{figure*}[tbp]
\centering
\resizebox{\textwidth}{!}{%
\begin{tikzpicture}[
  >=stealth, line width=0.8pt, font=\sffamily\footnotesize,
  hdr/.style={rounded corners=3pt, text=white, font=\sffamily\footnotesize\bfseries,
              minimum height=0.62cm, align=center, inner sep=3pt},
  st/.style={rounded corners=3pt, draw=violet!65, fill=violet!6, text width=6.6cm,
             align=left, font=\sffamily\small, minimum height=0.6cm, inner sep=4pt},
  num/.style={circle, fill=violet!75, text=white, font=\sffamily\scriptsize\bfseries,
              inner sep=1pt, minimum size=0.5cm, anchor=east},
  mo/.style={rounded corners=2pt, draw=violet!50, fill=violet!12, text width=1.75cm,
             align=center, font=\sffamily\scriptsize, minimum height=0.95cm, inner sep=2pt},
  ag/.style={circle, draw=blue!65, fill=blue!12, minimum size=0.95cm, inner sep=0pt},
  ex/.style={rounded corners=2pt, draw=green!45!black, fill=green!7, text width=4.15cm,
             align=left, font=\sffamily\scriptsize, inner sep=4pt, minimum height=0.98cm},
  flow/.style={->, line width=1pt, blue!70},
  pol/.style={->, line width=1pt, violet!75},
  esc/.style={->, dashed, line width=1pt, violet!70},
  oflow/.style={->, line width=1pt, green!55!black},
  vf/.style={->, line width=0.8pt, violet!65},
]

\begin{scope}[on background layer]
  \draw[rounded corners=6pt, draw=blue!60, fill=blue!3, line width=1pt]
        (-13.5,3.35) rectangle (-8.85,11.7);
  \draw[rounded corners=6pt, draw=violet!65, fill=violet!2, line width=1pt]
        (-6.15,2.3) rectangle (6.15,11.7);
  \draw[rounded corners=6pt, draw=green!50!black, fill=green!3, line width=1pt]
        (7.85,2.61) rectangle (12.55,11.7);
\end{scope}

\node[hdr, fill=blue!70, minimum width=4.4cm] at (-11.17,11.35)
     {Insured Agentic-AI Deployment};
\node[ag, minimum size=1.3cm, fill=blue!18, font=\sffamily\small] (orc) at (-11.17,8.4) {Orch.};
\node[ag] (a1) at (-11.17,10.20) {};
\node[ag] (a2) at (-12.88,8.96) {};
\node[ag] (a3) at (-12.23,6.94) {};
\node[ag] (a4) at (-10.11,6.94) {};
\node[ag] (a5) at (-9.46,8.96) {};
\foreach \a in {a1,a2,a3,a4,a5} {\draw[blue!65, line width=0.8pt] (orc) -- (\a);}
\draw[blue!40, dashed, line width=0.7pt] (a1)--(a2)--(a3)--(a4)--(a5)--(a1);
\node[font=\sffamily\small\itshape, text=blue!70] at (-11.17,5.9)
     {Interconnected AI agents};
\node[rounded corners=3pt, draw=blue!45, fill=blue!6, text width=4.15cm, align=center,
      font=\sffamily\footnotesize, minimum height=1.5cm] at (-11.17,4.35)
     {\textbf{External environment:}\\[2pt] APIs, databases, users, devices, third-party platforms};

\node[hdr, fill=violet!75, minimum width=12.0cm] at (0,11.35)
     {Automated Insurance Workflow Engine};

\node[st] (s1) at (0,10.5) {Onboarding \& risk submission};
\node[st] (s2) at (0,9.72) {Risk assessment \& underwriting};
\node[st] (s3) at (0,8.94) {Policy issuance};
\node[num] at ([xshift=-0.06cm]s1.west) {1};
\node[num] at ([xshift=-0.06cm]s2.west) {2};
\node[num] at ([xshift=-0.06cm]s3.west) {3};
\draw[vf] (s1) -- (s2);
\draw[vf] (s2) -- (s3);
\node[font=\sffamily\scriptsize, text=violet!55, anchor=west] at (3.55,9.72)
     {\begin{tabular}{@{}l@{}}Underwriting \&\\ policy creation\end{tabular}};

\draw[rounded corners=3pt, draw=violet!30, fill=violet!4] (-5.65,6.77) rectangle (5.65,8.42);
\node[font=\sffamily\scriptsize\itshape, text=violet!75!black] at (0,8.18)
     {Continuous Monitoring \& Risk Detection};
\node[mo] (m1) at (-4.48,7.5) {Runtime monitoring};
\node[mo] (m2) at (-2.24,7.5) {Anomaly detection};
\node[mo] (m3) at (0,7.5) {Policy trigger engine};
\node[mo] (m4) at (2.24,7.5) {Risk scoring};
\node[mo] (m5) at (4.48,7.5) {Alerting};
\draw[vf] (s3.south) -- (0,8.42);

\node[st] (s4) at (0,6.25) {Automated claim initiation};
\node[st] (s5) at (0,5.47) {Claim validation \& assessment};
\node[st] (s6) at (0,4.69) {Decision \& approval};
\node[st] (s7) at (0,3.91) {Settlement \& payout};
\node[num] at ([xshift=-0.06cm]s4.west) {4};
\node[num] at ([xshift=-0.06cm]s5.west) {5};
\node[num, fill=green!55!black] at ([xshift=-0.06cm]s6.west) {6};
\node[num, fill=green!55!black] at ([xshift=-0.06cm]s7.west) {7};
\draw[vf] (0,6.77) -- (s4.north);
\draw[vf] (s4) -- (s5);
\draw[vf] (s5) -- (s6);
\draw[oflow] (s6) -- (s7);

\node[rounded corners=3pt, draw=violet!45, fill=violet!5, text width=11.0cm, align=center,
      font=\sffamily\scriptsize, text=violet!80!black, minimum height=0.5cm] (ev) at (0,2.95)
     {\textbf{Key data \& evidence:} event logs, telemetry, audit trails, system context, external feeds, control posture};
\draw[<->, dotted, violet!55, line width=0.7pt] (-5.15,3.33) -- (-5.15,6.77)
     node[midway, rotate=90, font=\sffamily\scriptsize, text=violet!55, yshift=6pt] {feeds};

\node[hdr, fill=green!55!black, minimum width=4.5cm] at (10.2,11.35)
     {Human-in-the-Loop (Exceptions)};
\node[ex] at (10.2,10.25) {\textbf{Complex losses:} high-severity or systemic impacts};
\node[ex] at (10.2,8.87) {\textbf{Disputes:} coverage disagreements or conflicting evidence};
\node[ex] at (10.2,7.49) {\textbf{Fraud review:} suspicious behavior, collusion, abuse};
\node[ex] at (10.2,6.11) {\textbf{Coverage ambiguity:} unclear terms or novel scenarios};
\node[ex] at (10.2,4.73) {\textbf{High-severity events:} catastrophic or regulatory};
\node[ex, fill=green!12] at (10.2,3.35) {\textbf{Human decision:} recorded and fed back to the engine};

\draw[flow] (-8.85,9.4) -- (-6.15,9.4)
     node[midway, above, align=center, font=\sffamily\scriptsize, text=blue!70] {Risk state\\ $s_i$};
\draw[pol] (-6.15,8.0) -- (-8.85,8.0)
     node[midway, below, align=center, font=\sffamily\scriptsize, text=violet!75] {Policy \&\\ governance};
\draw[esc] (6.15,9.3) -- (7.85,9.3)
     node[midway, above, font=\sffamily\scriptsize, text=violet!70] {Escalation};
\draw[oflow] (7.85,7.9) -- (6.15,7.9)
     node[midway, below, font=\sffamily\scriptsize, text=green!55!black] {Feedback};

\end{tikzpicture}%
}
\caption{
Automated insurance workflow for insured agentic-AI systems. The workflow begins with onboarding and underwriting of insured AI agents characterized by the risk state \(s_i=(\alpha_i,\beta_i,\eta_i,g_i,v_i)\). Automated underwriting uses the risk-state, coverage, pricing, and optimization frameworks developed in Sections~\ref{sec:risk_coverage_framework} and \ref{sec:contract_design} to determine premiums, limits, deductibles, coverage allocations, and governance requirements. Once coverage is active, continuous monitoring, anomaly detection, trigger evaluation, and risk scoring observe the insured agent in real time. Covered events initiate automated claims processing, validation, indemnity computation, and settlement. Human experts intervene only for exceptional situations such as catastrophic losses, disputes, fraud investigations, coverage ambiguity, or regulatory concerns. The workflow operationalizes the mathematical programming framework by transforming insurance contracts into continuously monitored and computationally executable risk-transfer mechanisms.
}
\label{fig:automated_insurance_workflow}
\end{figure*}

The risk-state, coverage, pricing, and optimization frameworks developed above provide the economic and contractual foundation for agentic-AI insurance. Equations \eqref{eq:insured_cost}--\eqref{eq:coverage_bounds} define the underwriting variables, coverage structure, indemnity calculations, governance incentives, and optimization constraints that determine an economically feasible insurance contract. Figure~\ref{fig:automated_insurance_workflow} illustrates how these mathematical constructs can be operationalized through an automated agentic-AI insurance workflow.

Figure~\ref{fig:automated_insurance_workflow} should be read from left to right. The left panel represents the insured AI-agent network and the external environment with which it interacts; these interactions generate the underwriting inputs \((\alpha_i,\beta_i,\eta_i,g_i,v_i)\). The middle panel is the insurer-side workflow engine. Its upper stages transform the submitted risk state into underwriting, pricing, and policy-creation decisions, while its monitoring layer converts event logs, telemetry, audit trails, system context, external data feeds, and control-posture evidence into trigger evaluations and dynamic risk scores. The lower stages then move from trigger detection to claim initiation, validation, decision, approval, settlement, and payout. The right panel indicates that human involvement is not eliminated but reserved for exception classes, including complex losses, disputes, fraud review, coverage ambiguity, high-severity events, and final human decisions that feed back into the automated workflow. Throughout, the risk-state variables map to operational insurance outputs: autonomy category and operational authority primarily affect severity and liability, the permission profile defines feasible loss pathways and weighted exposure, governance maturity affects expected loss and incentive terms, and dependency concentration captures systemic accumulation risk.

A distinguishing feature of agentic-AI systems is that they emit continuous streams of operational evidence---telemetry, audit logs, execution traces, dependency-status information, permission changes, model-performance indicators, and governance-control metrics. Where traditional insurance leans on periodic audits and retrospective claims investigations, these signals allow continuous observation of the insured system throughout the policy period. Many underwriting, monitoring, and claims-management functions can therefore be delegated to specialized insurance agents acting on the insurer's behalf.

The workflow begins with agent onboarding and risk submission. Information describing the insured deployment is collected and mapped into the agentic-AI risk state \(s_i=(\alpha_i,\beta_i,\eta_i,g_i,v_i)\). This state representation serves as the primary underwriting input and determines the event probabilities, expected losses, governance requirements, and risk-loading components appearing in the pricing and optimization framework.

Automated underwriting then evaluates the submitted risk state and solves the contract-design problem defined by \eqref{eq:insurer_problem}--\eqref{eq:coverage_bounds}. The resulting contract specifies the premium \(T_i\), layer deductibles \(D_i^r\), layer limits \(L_i^r\), aggregate AI limit \(A_i\), binary coverage-incidence matrix \(\Gamma_i\), payment-share matrix \(\Lambda_i\), and governance obligations \(\Psi_i\). Policy issuance therefore becomes the execution of a computational underwriting process rather than a purely manual assessment activity.

Once coverage is active, continuous monitoring modules observe the insured agent and its operating environment. Runtime monitoring, anomaly detection, policy-trigger evaluation, and dynamic risk scoring continuously update the insurer's estimate of exposure. Observed events are mapped into the event space \(\mathcal E\) introduced earlier, allowing the system to determine whether a covered loss event has occurred. Because the event taxonomy and coverage allocations are already embedded within the contract, trigger evaluation can be performed automatically.

When a trigger condition is satisfied, an automated claims agent initiates the claims process. Evidence is collected from telemetry records, audit logs, dependency-status information, system context, and external data sources. Claim validation determines the applicable event category, estimates the realized loss \(x_i^e\), evaluates exclusions and coverage conditions, and computes the corresponding indemnity payment using \eqref{eq:indemnity}. For routine claims, settlement can be executed automatically once contractual requirements have been verified.

\subsection{Discrete-Event Simulation of Online Claims Execution}

To make the workflow operational, we implement a discrete-event simulator for the designed healthcare contract in Section~\ref{sec:designed_healthcare_contract}. The issued contract is fixed throughout the simulation: \(T_i^\star=\$54,833\), \(D_i^r=\$25,000\), \(L_i^r=\$500,000\), \(A_i=\$1,500,000\), the coverage schedule is the event-layer assignment in Table~\ref{tab:issued_healthcare_layer_schedule}, and the policy covenants fix the endorsed permission boundary, authority threshold, logging requirements, incident notice requirement, and governance tier \(g_i^\star=g^{(3)}\). The simulation therefore does not re-optimize the contract. To make the online process visible without making a single insured look unrealistically loss-prone, we simulate a small portfolio of 25 comparable healthcare deployments, each issued on the same Section~\ref{sec:designed_healthcare_contract} terms. The per-contract claim-arrival process remains calibrated to the status-quo frequencies used in the healthcare case study: the six annual event probabilities sum to 0.217 material events per contract-year, and the contract-design calculation gives expected indemnity of \$35,833 per contract-year. Across the portfolio, this corresponds to 5.4 expected material claims and \$895,825 of expected indemnity per year.

The simulator maintains an event queue whose elements are time-stamped workflow states: telemetry alert, automated detection, claim validation, human review, and claim closure. Each alert carries an event class \(e\), a realized loss estimate \(x_i^e\), a detection score, an evidence score, an ambiguity score, and indicators for permission-boundary or covenant violations. An alert becomes an automated detection when its detection score exceeds the trigger threshold. A detected material event becomes a claim notice. The validation agent denies claims that fall outside the endorsed permission boundary, violate a governance covenant, or lack sufficient evidence. It automatically pays claims with strong evidence and low ambiguity. It escalates claims with high ambiguity, cyber-physical consequences, or large projected payments. For an approved claim under contract \(j\) arriving at time \(t\), the operational payment is charged to that contract's remaining annual aggregate as \(Y_{j,t}=\min\{(x_{j,t}-D)^+,L,A-\sum_{\tau<t}Y_{j,\tau}\}\). Thus each contract has its own aggregate, while the figure reports the portfolio-level sum of validated payments.

\begin{table}[H]
\centering
\caption{Discrete-Event Simulation of Automated Claims Execution}
\label{tab:automated_claim_simulation}
\begin{tabular}{ll}
\hline
Quantity & Simulated value \\
\hline
Portfolio size & 25 contracts \\
Expected material event frequency & 0.217 per contract-year \\
Expected portfolio claim count & 5.4 per year \\
Expected portfolio indemnity & \$895,825 \\
Telemetry alerts processed & 102 \\
Automated detections & 25 \\
Claim notices opened & 16 \\
Automatically paid claims & 9 \\
Human-reviewed claims & 0 \\
Denied claims & 7 \\
Total indemnity paid & \$1,303,776 \\
Remaining portfolio aggregate capacity & \$36,196,224 \\
Portfolio aggregate capacity used & 3.48\% \\
Median claim closure time & 1.3 days \\
\hline
\end{tabular}
\end{table}
 
\begin{figure}[p]
\centering
\includegraphics[width=\textwidth]{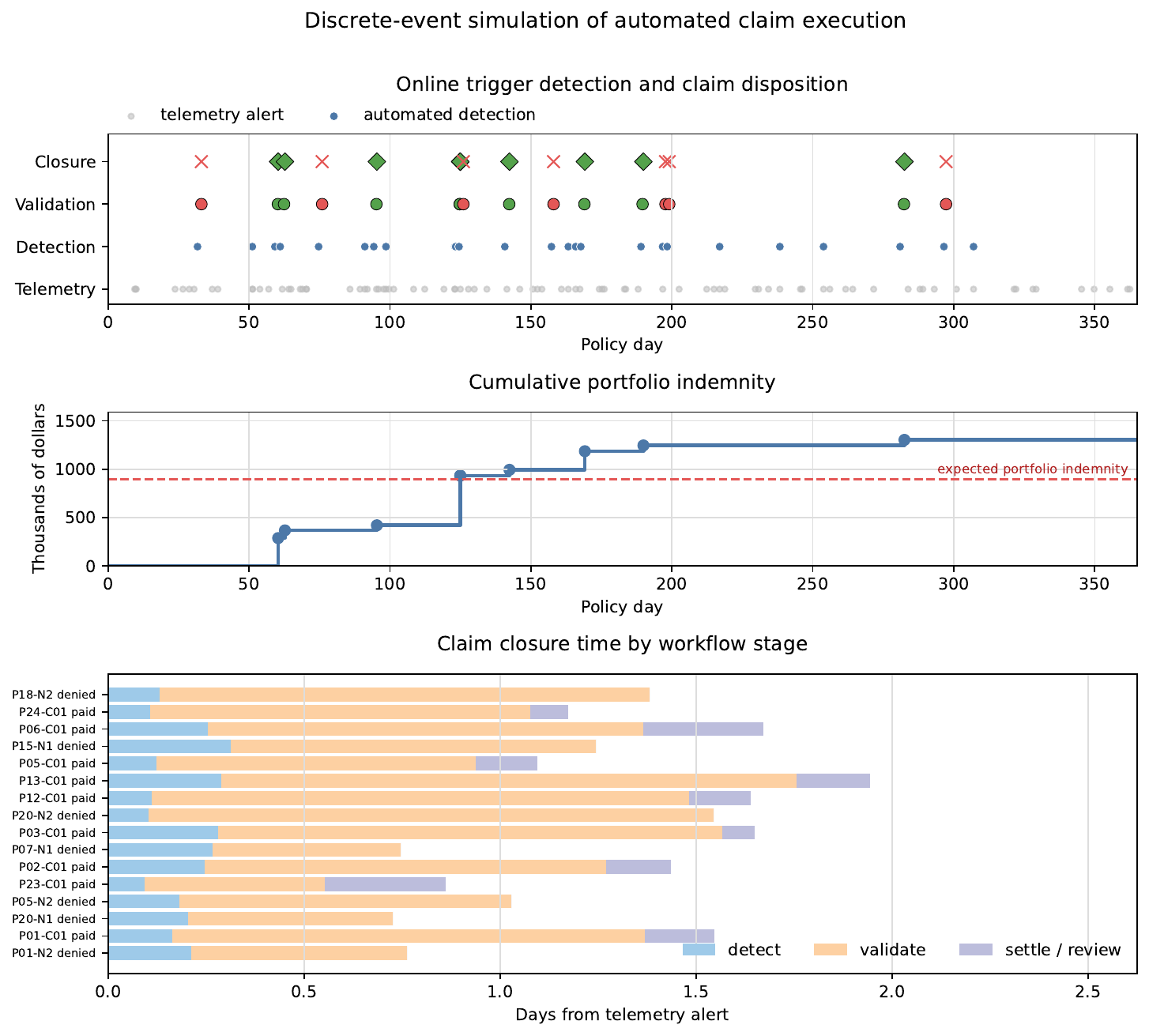}
\caption{Discrete-event simulation of online automated claim execution for a small portfolio of designed healthcare contracts. The panels show trigger disposition, cumulative portfolio indemnity relative to expected indemnity, and claim closure time by workflow stage.}
\label{fig:automated_claim_simulation}
\end{figure}

Figure~\ref{fig:automated_claim_simulation} illustrates one portfolio-year realization. The monitoring layer processes 102 telemetry alerts across 25 contracts. Twenty-five alerts exceed the automated-detection threshold and 16 become claim notices. Nine covered claims are validated and paid automatically, while seven notices are denied because they lack sufficient evidence or fall outside the endorsed permission boundary. Total indemnity paid is \$1,303,776, compared with \$895,825 of expected portfolio indemnity, leaving \$36,196,224 of remaining portfolio aggregate capacity and using 3.48\% of available aggregate capacity. The realized payment is above the actuarial expectation but remains credible for an active portfolio year because the per-contract frequency is unchanged and the additional events arise from portfolio breadth rather than inflated single-policy frequency. The example shows how the same mathematical contract can be executed as a monitored online process: telemetry produces triggers, triggers produce validation tasks, validation produces automatic payment, denial, or escalation, and every approved payment updates the corresponding contract aggregate.

Human participation is reserved primarily for exceptional situations. As shown in Figure~\ref{fig:automated_insurance_workflow}, escalation occurs when the automated workflow encounters disputed causation, suspected fraud, ambiguous coverage allocation, catastrophic losses, regulatory concerns, or other circumstances requiring human judgment. Human decisions are subsequently fed back into the workflow, enabling continual refinement of underwriting models, claims procedures, governance standards, and risk-assessment mechanisms.

From this vantage, the mathematical-programming framework and the automated workflow are complementary halves of one system. The optimization model determines \emph{what} contract to offer through the solution of \eqref{eq:insurer_problem}, and the workflow determines \emph{how} that contract is monitored, enforced, and executed across the policy lifecycle. Automated agentic-AI insurance is feasible precisely because these two capabilities meet: the formal contract-design framework of Equations \eqref{eq:insured_cost}--\eqref{eq:coverage_bounds}, and the ability of agentic-AI systems to supply real-time observability, automated reasoning, evidence collection, risk computation, and claims execution.

\section{Conclusion}

Agentic AI pushes insurance beyond asset-centered cyber coverage toward behavior-centered risk transfer. The relevant exposure is not whether an organization uses AI, but what the deployed agent is authorized to do, which external systems it can affect, how often actions execute without human approval, how governance controls are evidenced, and how concentrated the deployment is in shared model, cloud, data, and connector dependencies. This paper formalizes these underwriting questions through the risk state \(s_i=(\alpha_i,\beta_i,\eta_i,g_i,v_i)\) and connects that state to event probabilities, loss severities, coverage incidence, indemnity allocation, governance covenants, risk loadings, and contract-design constraints.

The main contribution is a unified mathematical architecture for AI-native insurance. The framework separates event occurrence from coverage-layer allocation through \(\Gamma_i\) and \(\Lambda_i\), treats governance covenants \(\Psi_i\) as enforceable policy obligations rather than informal underwriting notes, and formulates the insurer's contract problem over premiums, deductibles, limits, aggregate exposure, allocation rules, and governance requirements. On this foundation, three structural results (Definition~\ref{def:insurability_region} and Propositions~\ref{prop:monotone_feasibility}--\ref{prop:governance_certification}) show that insurability is a region of the underwriting state space, that fixed-terms premium feasibility deteriorates monotonically as exposure increases under monotone pricing primitives, and that a governance threshold certifies a deployment as insurable at given coverage terms.

The framework also gives a policy interpretation of insurance as an AI operating cost and regulatory-control mechanism. If insurance is bundled into the cost of AI deployment or mandated for high-risk systems, the premium functions like a risk price: it can discourage deployments whose private benefits do not justify their expected harms, while requiring financial responsibility for systems that can affect patients, customers, infrastructure, or the public. This creates a calibration problem for regulators and insurers. A mandate that is too broad may operate as a general AI tax and deter low-risk innovation, while a mandate that is too narrow may leave high-impact agentic systems without adequate governance or loss-absorption capacity.

The healthcare care-coordination case study demonstrates the framework in practice, translating operational facts about a clinical agent into event probabilities, gross losses, indemnity payments, risk loadings, participation bounds, and feasible premiums. The numerical comparisons confirm the structural results: higher delegated authority, broader permission exposure, and stronger dependency concentration narrow or eliminate the feasible premium interval, while stronger governance lowers expected loss and restores insurability. This is the central economic role of governance-sensitive pricing---monitoring, approval gates, logging, testing, rollback capability, and notice obligations are not compliance details but contractible controls that shape both risk and marketability.

Several extensions remain important. First, empirical calibration will require claims data, incident reports, audit evidence, telemetry, and stress scenarios specific to agentic-AI deployments. Second, portfolio models are needed to capture accumulation risk from shared model providers, cloud platforms, data sources, and agent frameworks. Third, dynamic policy mechanisms should adjust pricing, deductibles, limits, and covenants as agents gain permissions, change autonomy modes, or migrate dependencies. Finally, reinsurance and capital models will be needed for correlated failures that strike many insured agents at once. Together these steps would carry agentic-AI insurance from a contract-design framework toward a full actuarial and market infrastructure for governing autonomous AI risk.
 
\bibliographystyle{abbrv}

\end{document}